\documentclass[10pt,journal,compsoc]{IEEEtran}
\usepackage[pdftex]{graphicx}
\usepackage{subfigure}
\usepackage{amsmath,amssymb,amsfonts}
\usepackage{color}
\usepackage{blindtext}
\usepackage{algorithm}
\usepackage{algorithmic}
\usepackage{cite}
\usepackage{amsthm}
\usepackage[normalem]{ulem}
\usepackage{url}
\usepackage{comment}
\usepackage{multirow}
\usepackage{xcolor}
\usepackage{diagbox}
\usepackage{indentfirst}

\newtheorem{thm}{Theorem}

\newtheorem{definition}{Definition}
\usepackage{hyperref}
\newcounter{ToDo}
\newcounter{gaocomm}
\newcounter{Note}
\definecolor{blue-violet}{rgb}{0.54, 0.17, 0.89}
\definecolor{mygreen}{rgb}{0.0, 0.5, 0.0}
\definecolor{awesome}{rgb}{1.0, 0.13, 0.32}
\definecolor{bostonuniversityred}{rgb}{0.8, 0.0, 0.0}

\begin{document}

\newcommand{\point}{
    \raise0.7ex\hbox{.}
    }

\title{Partial Sum Minimization of Singular Values Representation on Grassmann Manifolds} 

\author{Boyue~Wang, 
        Yongli~Hu~\IEEEmembership{Member,~IEEE,} Junbin~Gao, Yanfeng~Sun~\IEEEmembership{Member,~IEEE,}
        and Baocai~Yin  ~\IEEEmembership{Member,~IEEE} 
\IEEEcompsocitemizethanks{
\IEEEcompsocthanksitem  Boyue Wang, Yongli Hu and Yanfeng Sun are with Beijing Municipal Key Lab of Multimedia and Intelligent Software Technology, College of Metropolitan Transportation, Beijing University of Technology, Beijing 100124, China. 
E-mail: boyue.wang@emails.bjut.edu.cn, \{huyongli, yfsun\}@bjut.edu.cn
\IEEEcompsocthanksitem Junbin Gao is with the Discipline of Business Analytics, The University of Sydney Business School, The University of Sydney, NSW 2006, Australia. \protect E-mail: junbin.gao@sydney.edu.au 
\IEEEcompsocthanksitem Baocai Yin is with the College of Computer Science and Technology, Faculty of Electronic Information and Electrical Engineering, Dalian University of Technology, Dalian 116620, China; and with Beijing Municipal Key Lab of Multimedia and Intelligent Software Technology at Beijing University of Technology, Beijing 100124, China. \protect E-mail: ybc@bjut.edu.cn
}
}


\IEEEcompsoctitleabstractindextext{%
\begin{abstract}
Clustering is one of the fundamental topics in data mining and pattern recognition. As a prospective clustering method, the subspace clustering has made considerable progress in recent researches, e.g., sparse subspace clustering (SSC) and low rank representation (LRR).
However, most existing subspace clustering algorithms are designed for vectorial data from linear spaces, thus not suitable for high dimensional data with intrinsic non-linear manifold structure. For high dimensional or manifold data, few research pays attention to clustering problems. The purpose of clustering on manifolds tends to cluster manifold-valued data into several groups according to the mainfold-based similarity metric. This paper proposes an extended LRR model for manifold-valued Grassmann data which incorporates prior knowledge by minimizing partial sum of singular values instead of the nuclear norm, namely Partial Sum minimization of Singular Values Representation (GPSSVR). The new model not only enforces the global structure of data in low rank, but also retains important information by minimizing only smaller singular values. To further maintain the local structures among Grassmann points, we also integrate the Laplacian penalty with GPSSVR.
The proposed model and algorithms are assessed on a public human face dataset, some widely used human action video datasets and a real scenery dataset. The experimental results show that the proposed methods obviously outperform other state-of-the-art methods.
\end{abstract}
}

\maketitle

\section{Introduction}

In fact, with the wide use of cheaper cameras
in many domains such as human action recognition, safety production detection and traffic jam detection, there are huge amount of video data that need to be processed efficiently. However, it is impossible to deal with so many videos with very limited labels.  
Thus unsupervised video clustering algorithms have attracted increasing interests recently  \cite{TuragaVeeraraghavanSrivastavaChellappa2011,ShiraziHarandiSandersonAlaviLovell2012,WangHuGaoSunYin2016} and it is urgently desired to achieve good clustering performance for real world videos. To achieve this goal, it is critical to explore a proper similarity for high dimensional data and build a proper clustering model based on the new similarity and new representation.

\begin{figure}
    \begin{center}
    \includegraphics[width=0.5\textwidth]{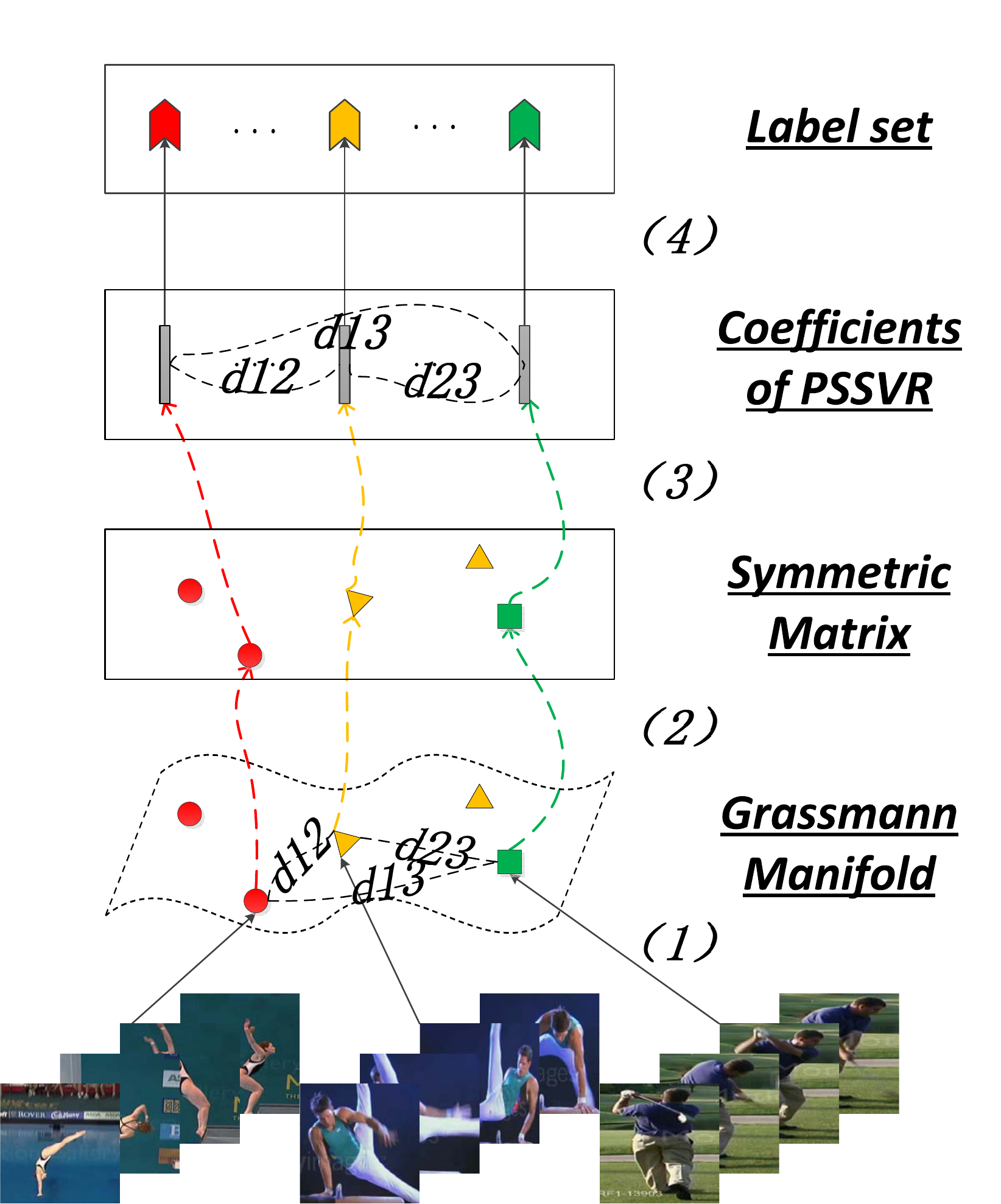}
    \end{center}
    \caption{(1) Image sets are represented as Grassmann points. (2) All Grassmann points are mapped into the symmetric matrices. (3) PSSVR model is formulated in symmetric matrix space and we constrain the coefficient matrix maintaining the inner structure of origin data. (4) Clustering by NCuts.}\label{Fig1}
\end{figure}

Research on clustering has made great progress in the last few decades, especially subspace clustering. However, most existing subspace clustering algorithms, e.g. SSC and LRR, are designed in Euclidean space, which are not suitable to perform high dimensional data with intrinsic non-linear manifold structure. In this paper, to explore the clustering problem of high dimension data, we intend to extend the classic LRR model onto manifold space. In the following, we will briefly review the original LRR and discuss its some limitations.

LRR has become one of the most successful self-expressive models for clustering vectorial data according to their subspace structures \cite{LiuLinSunYuMa2013,Guo2015,YinGaoLin2015,ZhangXuLiSun2015,YinGaoLinShiGuo2015,WangHuGaoSunYin2016}.
The core idea in the original LRR is based on the Rank Minimization principle which results in a non-convex problem. To provide a practical implementation for LRR, one employs the nuclear norm as a surrogate of the Rank Minimization regularization.

However, indirectly minimizing the rank of the coefficient matrix by minimizing the nuclear norm is not a perfect approximation way. The main argument is that minimizing the nuclear norm is equivalent to minimizing the sum of all the singular values of the affinity matrix. This strategy ignores the fact that different singular values of the affinity/similarity matrix generally correspond to different importance.

Actually, the larger the singular value is, the more energy the corresponding singular vector contains. So concentrating energy into several larger singular values benefits for clustering or classification via reducing the rank of the affinity matrix. Inspired by this motivation, Truncated Nuclear Norm (TNN) \cite{HuZhangYeLiHe2013} and Partial Sum minimization of Singular Values (PSSV) \cite{OhTaiBazinKimKweon2015} both propose maintaining several largest singular values unchanged and minimizing the rest (PSSV norm). Doing so achieves better performance than applying Robust Principal Component Analysis (RPCA) \cite{CandesLiMaWright2011} in image recovery applications.

In this paper, we intend to replace the nuclear norm in LRR by the PSSV norm to construct a new clustering model, called Partial Sum Minimization of Singular Values Representation (PSSVR) model. Compared with LRR, PSSVR is not only able to capture global structures of the data, but also to take into account the prior knowledge of the practical applications.

We also note that LRR or other clustering methods are designed for vectorial data which are generated from linear spaces and  the dissimilarity of data is measured by Euclidean distance.  This has limited the application of LRR in handling with very high dimensional data, due to high computational cost, such as large scale image sets and video data. Additionally, it has been proven that such high dimensional data are always embedded in nonlinear low dimension manifold \cite{WangShanChenGao2008} and it is inappropriate to use the current LRR method to handle them.

Most existing manifold learning methods, which explore the nonlinear manifold structure hidden in high dimension data \cite{RoweisSaul2000,TenenbaumSilvaLangford2000,HeNiyogi2003,ZhangZha2004,HuangNieHuangDing2014,ChangNieYangZhangHuang2016}, are mainly designed for vectorial data usually with high computational complexity and are not suitable to process videos from wild practical sensors. Simply vectoring a video will destroy the spatio-temporal information and generate an ultra-high dimensional vector. Fortunately, Grassmann manifold is widely used to represent videos in recent research, see \cite{TuragaVeeraraghavanSrivastavaChellappa2011,HarandiSandersonShiraziLovell2011,HarandiSandersonShenLovell2013,WangHuGaoSunYin2016CSVT}.
In these strategies, a video clip is represented as a subspace, i.e., a point on Grassmann manifold. One of good properties of Grassmann manifold is that Grassmann manifold can be easily embedded into a linear space --- Symmetric matrix space. Therefore, all abstract Grassmann points are embedded into the symmetric matrix space and the clustering methods can be applied in the embedded space for Grassmann manifolds. Utilizing these advantages of Grassmann manifolds, we represent the high dimensional videos or image sets as Grassmann points for clustering.

\begin{definition}[Clustering on Grassmann Manifold] Given  a set of `points' on a given Grassmann manifold, i.e., a number of subspaces of same dimension in an Euclidean space, the task of clustering on Grassmann manifolds is to cluster all the given `points' (subspaces) into several groups under a similarity measure.
\end{definition}

For example, on 2D Euclidean space (the plane), we are given a set of lines passing through the origin, i.e., `points' (subspaces) on the Grassmann manifold $\mathcal{G}(2,1)$. We cluster them into their respective groups. Fig. \ref{Fig2} briefly shows the behavior of clustering on Grassmann manifold.

\begin{figure*}
    \begin{center}
    \includegraphics[width=0.7\textwidth]{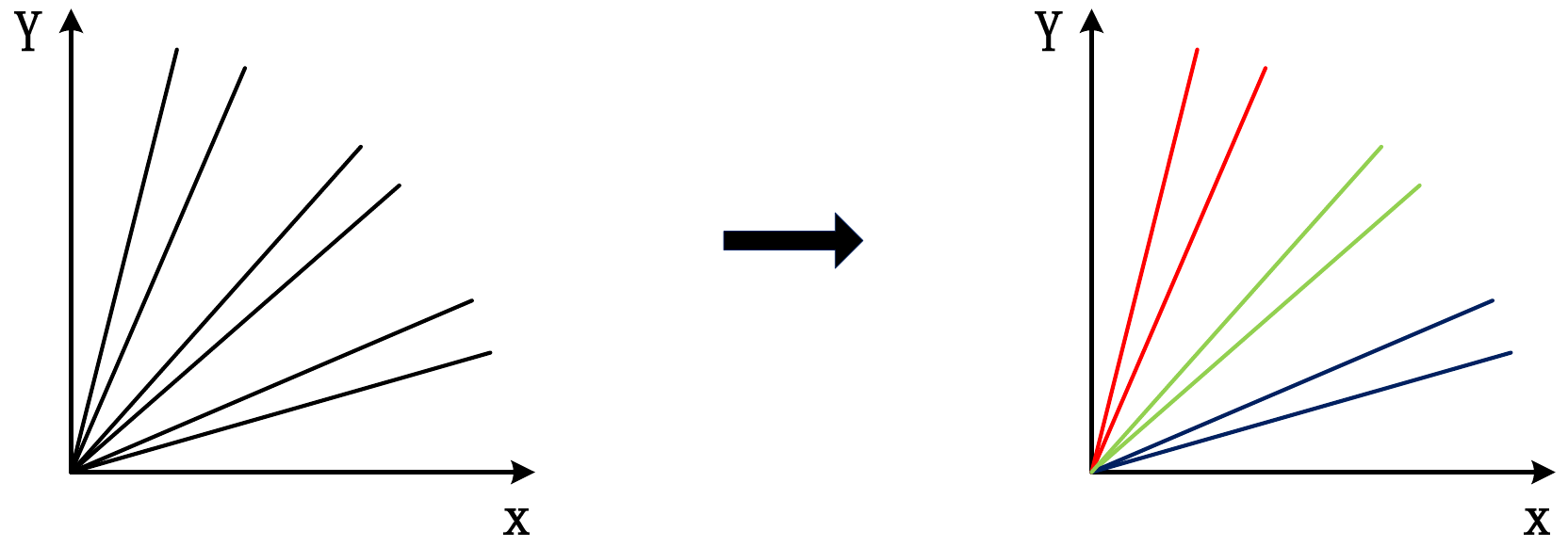}
    \end{center}
    \caption{Given a set of 1-dimensional subspaces in $R^2$, or 'points' on Grassmann manifold $\mathcal G(2,1)$ and cluster these 'points' (subspaces) into their three respective groups.}\label{Fig2}
\end{figure*}

In this paper, we combine Grassmann manifolds with the above PSSVR,  leading to a new clustering method, namely Grassmann manifolds PSSVR model (GPSSVR). The whole clustering procedure is illustrated in Fig. \ref{Fig1}. The videos or image sets are firstly represented as Grassmann points. All Grassmann points are then embedded into the symmetric matrix space as mentioned above, thus we naturally extend the PSSVR model in Euclidean space to the one on Grassmann manifolds. GPSSVR explores the intrinsic nonlinear relation hidden in high dimensional data and implements clustering on the manifold. 

GPSSVR mainly reveals the global structure underlying the data while the local structural information of the data is not well considered. To address this limitation, we further introduce a local structure constraint, based on Laplacian matrix, into our model to model the local feature of the data, resulting in the Laplacian GPSSVR, named as LapGPSSVR.  The contribution of this paper is summarized as follows:

\begin{itemize}
\item Proposing a novel form of LRR model, so-called PSSVR, which not only extracts the low rank structure of data, but also retains important information for clustering.
\item Extending the PSSVR model onto the Grassmann Manifolds based on our previous works in \cite{WangHuGaoSunYin2014,WangHuGaoSunYin2016}; and giving a practical solution to the proposed GPSSVR model; and
\item Introducing a Laplacian matrix based constraint into the GPSSVR model to represent the local geometry of the data. 
\end{itemize}

The rest of the paper is organized as follows. In Section \ref{Sec:RelatedWork}, we review some related works. In Section \ref{Sec:2}, we summarize the geometric properties of Grassmann manifolds and some basic knowledge of LRR and PSSV. In Section \ref{Sec:3}, we propose the PSSVR on Grassmann manifolds and detail the solution to it. In Section \ref{Sec:4}, a Laplacian constraint is introduced into our proposed model to maintain the local structure of data. In Section \ref{Sec:5}, the performance of the proposed methods are evaluated on several public datasets. In the last section, we conclude the paper and elaborate our future work.

\section{Related work}\label{Sec:RelatedWork}

In this section, we review in more details on several classical subspace clustering algorithms for linear subspaces, a number of methods with the improved nuclear norms, and some manifolds in literature.

Clustering is a fundamental problem in computer vision and machine learning. A large number of methods have been proposed to solve this problem, such as the conventional iterative methods \cite{Tseng2000,HoYangLimLeeKriegman2003}, the statistical methods \cite{TippingBishop1999a,GruberWeiss2004}, the factorization-based algebraic approaches \cite{Kanatani2001,HongWrightHuangMa2006,MaYangDerksenFossum2008}, and the spectral clustering methods \cite{ChenLerman2009,ElhamifarVidal2009,LiuLinYu2010,LiuYan2011,FavaroVidalRavichandran2011,LangLiuYuYan2012,LiuLinSunYuMa2013,ZhangZongYouYong2016}. Among all the clustering methods, the Spectral Clustering (SC) algorithm is state-of-the-art with excellent performance in many applications \cite{Luxburg2007,ElhamifarVidal2013} by exploring affinity information of data. In this framework, final clustering is obtained by applying a spectral method such as the Normalized Cuts (NCut) \cite{ShiMalik2000} on the affinity matrix learned from data.  As a result, how to construct an effective affinity matrix becomes the key question. Since our method belongs to this type of spectral clustering algorithms, we review the related works along this direction.

Two representative methods are Sparse Subspace Clustering (SSC) \cite{ElhamifarVidal2013} and Low Rank Representation (LRR) \cite{LiuLinSunYuMa2013}, which are both based on the data self-expressive property. SSC uses the sparsest self representation of data produced by $l_1$-norm to construct the affinity matrix, while LRR relies on the Rank Minimization regularization, inspired by RPCA. Different from SSC, which only independently focuses on the sparsest representation for each datum and ignores the relations among object data, LRR explores the matrix rank to capture the underlying global structure hidden in a data set. It has been proven that, when a data set is actually from a union of several low-dimension subspaces, LRR can reveal this structure to facilitate subspace clustering \cite{LiuLinSunYuMa2013}. In many clustering scenarios, LRR has obtained successful applications, such as face recognition \cite{ZhangJiangDavis2013}, visual tracking \cite{ZhangGhanemLiuAhuja2012} and saliency detection \cite{LangLiuYuYan2012}.

LRR employs the nuclear norm to approximate to the Rank Minimization regularization.
Actually, in many applications, the rank of data (the matrix $A$ in formula (\ref{PSSVModel})) is known, for example, 3 in  photometric stereo application and 1 in background subtraction, but in the current minimization of nuclear norm, this prior information has not been well utilized. 

To address the issues associated with the nuclear norm, researchers propose some non-convex penalty functions which are better approximation to the rank minimization 
and easier to optimize. Gu \emph{et al}. \cite{GuZhangZuoFeng2014} propose the weighted nuclear norm minimization method (WNNM). Jeong and Lee \cite{JeongLee2014} use the Schatten $p$-norm of the singular values to fit the rank minimization of a matrix. Most interestingly, Hu \emph{et al}. \cite{HuZhangYeLiHe2013} and Oh \emph{et al}. \cite{OhTaiBazinKimKweon2015} both propose minimizing only the smallest $N-r$ singular values while keeping the largest $r$ singular values unconstrained, where $N$ is the number of singular values of the matrix and $r$ is the expected rank of the matrix which could be usually estimated by using prior knowledge.

In order to explore the nonlinear manifold structure hidden in high dimensional data and obtain their proper representation \cite{HuangNieHuangDing2014,LuLaiFanCuiZhu2015,ChangNieYangZhangHuang2016}, many manifold learning methods are proposed, such as Locally Linear Embedding (LLE) \cite{RoweisSaul2000}, ISOMAP \cite{TenenbaumSilvaLangford2000}, Locally Linear Projection (LLP) \cite{HeNiyogi2003}, and Local Tangent Space Alignment (LTSA) \cite{ZhangZha2004}. However, these methods are usually designed to handle vectorial data with higher computational complexity, which are not suitable to high dimensional data, such as videos. As each of those videos may contains different number of frames even without correct temporal relation, simply vectorizing them may produce vectors in different dimensions. Given the shortcomings of video vector representation, Grassmann manifolds has become a competitive tool \cite{TuragaVeeraraghavanSrivastavaChellappa2011,HarandiSandersonShenLovell2013}.

\section{Background Theory}\label{Sec:2}

We review some concepts about Grassmann Manifolds, Low Rank Representation and Partial Sum Minimization of Singular Values, which pave the way for introducing the proposed method.

\subsection{Grassmann Manifolds}
Grassmann manifolds $\mathcal{G}(p,d)$ \cite{AbsilMahonySepulchre2008} consists of the set of all linear $p$-dimensional subspaces of $\mathbb{R}^d$ $(0\leq p\leq d)$, which can be represented by the quotient space of all the $d\times p$ matrices with $p$ orthogonal columns under the $p$-order orthogonal group $\mathcal{O}(p)$:
\[
\mathcal{G}(p,d) = \{\mathbf X\in \mathbb{R}^{d\times p} : \mathbf X^T\mathbf X = \mathbf I_p\} / \mathcal{O}(p).
\]
there are two popular methods to define a metric on Grassmann manifold: one is to define consistent metrics in tangent spaces to make Grassmann manifold a Riemann manifold (intrinsic metric) \cite{GohVidal2008,CetingulWrightThompsonVidal2014}; and the other one is to embed Grassmann manifold into the symmetric matrix space where the Euclidean distance (Fibonacci norm) can be applied. The latter one is easier and more effective in practice, and the mapping relation can be represented as \cite{HarandiSandersonShenLovell2013},
\begin{equation}
\label{Grassmann2Sym_mapping}
  \Pi : \mathcal{G}(p,d) \rightarrow \text{Sym}(d), \ \ \ \Pi(\mathbf X)=\mathbf X\mathbf X^T.
\end{equation}
The embedding $\Pi(\mathbf X)$ is diffeomorphism \cite{HelmkeHuper2007}. In this paper, we adopt the second strategy on Grassmann manifolds to define the following distance inherited from the symmetric matrix space under this mapping \cite{HarandiSandersonShenLovell2013},

\begin{equation}
\label{Dist_Grassmann}
 d^2_g(\mathbf X,\mathbf Y) = \frac12\|\Pi(\mathbf X)-\Pi(\mathbf Y)\|^2_F. 
\end{equation}

A point on Grassmann manifolds is actually an equivalent class of all the orthogonal matrices in $\mathbb{R}^{d\times p}$, any one of which can be converted to the other by a $p\times p$ orthogonal matrix. Thus Grassmann manifolds is naturally regarded as a good representation for video clips/image sets, thus can be used to tackle the problem of videos matching.

\subsection{Low Rank Representation}
Given a set of data drawn from an unknown union of subspaces $\mathbf X = [\mathbf x_1, \mathbf x_2, ..., \mathbf x_m]\in\mathbb{R}^{d\times m}$ where $d$ is the data dimension and $m$ is the number of data, the objective of subspace clustering is to assign each data sample to its underlying subspace. The basic assumption is that the data in $\mathbf X$ are drawn from a union of $k$ subspaces $\{\mathcal{S}_i\}^k_{i=1}$ of dimensions $\{d_i\}^k_{i=1}$.

Under the data self representation principle, each data point in the dataset can be written as a linear combination of other data points, i.e., $\mathbf X=\mathbf X\mathbf Z$, where $\mathbf Z\in\mathbb{R}^{m\times m}$ is a matrix of similarity coefficients.

The LRR model is formulated as \cite{LiuLinYu2010}
\begin{equation}
\label{LRRModel}
\min_{\mathbf Z, \mathbf E}\|\mathbf E\|^2_{F} + \lambda \|\mathbf Z\|_*, \ \text{ s.t. } \ \mathbf X = \mathbf X\mathbf Z+\mathbf E.
\end{equation}
In problem (\ref{LRRModel}), $\mathbf E$ is the error resulting from the self representation. $F$-norm can be changed to other norms e.g. $\ell_{2,1}$-norm as done in the original LRR model. When the data set does not contain many outliers, the final clustering accuracies have little differences between using $F$-norm and $\ell_{2,1}$-norm. What is more, problem (\ref{LRRModel}) has a closed-form solution which is faster many times than the one with the $\ell_{2,1}$-norm. LRR takes a holistic view in favor of a coefficient matrix in the lowest rank, measured by the nuclear norm $\|\cdot\|_*$, which uses the sum of all the singular values of the matrix to approximate to the Rank Minimization regularization.

\subsection{Partial Sum Minimization of Singular Values in RPCA (PSSV)}

To recover a low rank matrix $\mathbf A$ from corrupted data $\mathbf X$, RPCA \cite{WrightGaneshRaoPengMa2009} minimizes the  rank of matrix $\mathbf A$ by formulating the following problem,
\[
\min\limits_{\mathbf A,\mathbf E}\|\mathbf A\|_*+\lambda\|\mathbf E\|_0 \ \ \text{s.t.} \ \ \mathbf X=\mathbf A+\mathbf E,
\]
where $\mathbf E\in R^{d\times m}$ is the noise which is assumed to be sparse in RPCA model. If the data is corrupted by Gaussian noise, the $\ell_0$-norm can be replaced by Frobenius norm. However, in many practical problems, where the rank can be estimated, the nuclear norm limits model performance due to over-relaxing the rank minimization constraints and ignoring the individual importance of each singular value of matrix $\mathbf A$. Based on the prior knowledge about the rank of matrix $\mathbf A$, Oh \emph{et al}. \cite{OhTaiBazinKimKweon2015} propose a new model to minimize partial sum of singular values of matrix $\mathbf A$ while maintaining the rest singular values unconstrained, as defined by the following problem,
\begin{equation}
\label{PSSVModel}
\min\limits_{\mathbf A,\mathbf E} \|\mathbf A\|_{>r} + \lambda\|\mathbf E\|_1, \ \ \text{s.t.} \ \ \mathbf X=\mathbf A+\mathbf E,
\end{equation}
where PSSV norm $\|\mathbf A\|_{>r}=\sum\limits_{i=r+1}^{\min(d,m)}\sigma_i (\mathbf A)$ and $\sigma_i (\mathbf A)$ represents the $i$-th singular value of the matrix $\mathbf A$. The $r$ is the expected rank of the matrix $\mathbf A$ which may be derived from the prior knowledge of a defined problem.

\section{Partial Sum Minimization of Singular Values Representation on Grassmann Manifolds (GPSSVR)}\label{Sec:3}
In this section, we will propose an improved Rank Minimization approximation-based subspace clustering method, namely Partial Sum Minimization of Singular Values Representation (PSSVR), and extend it onto Grassmann manifolds. An effective solution to the proposed model on Grassmann manifolds is explored.

\subsection{PSSVR on Grassmann}
For a set of samples $\mathbf X=[\mathbf x_1, \mathbf x_2, ..., \mathbf x_m]\in R^{d\times m}$, we adopt the self-representation method same as LRR to represent the data and replace the nuclear norm of LRR with the partial sum of singular values of coefficient matrix. Thus we construct a 
PSSVR as follows:
\begin{equation}
\label{PSSVR-1}
\min\limits_{\mathbf Z} \|\mathbf Z\|_{>r}+ \lambda\|\mathbf E\|_F^2 \ \ \text{s.t.} \ \ \mathbf X = \mathbf X\mathbf Z + \mathbf E.
\end{equation}
Analogue to the low rank constraint on matrix $\mathbf A$, which represents the hidden clean data of the origin data $\mathbf X$, as the PSSV model shown in (\ref{PSSVModel}), we add low rank constraint on the coefficient matrix $\mathbf Z$ to reveal the low rank structure hidden in the origin data $\mathbf X$. In addition, we use $\|\cdot\|_F^2$ instead of $\|\cdot\|_1$ in (\ref{PSSVModel}) to measure the reconstruct error $\mathbf E$. By eliminating variable $\mathbf E$, we can write out an equivalent problem as follows
\begin{equation}
\label{PSSVR-2}
\min\limits_{\mathbf Z} \|\mathbf Z\|_{>r}+ \lambda\sum\limits_{i=1}^{m}\|\mathbf{x}_i-\sum_{j=1}^m z_{ij}\mathbf{x}_j\|_F^2,
\end{equation}
where the measure $\|\mathbf{x}_i-\sum\limits_{j=1}^m  z_{ij}\mathbf{x}_j\|_F^2$ is the Euclidean distance between the point $\mathbf x_i$ and its linear combination of all the other data points including $\mathbf x_j$ and $\mathbf Z = [z_{ij}]$.

Now, let us consider the generalization of problem (\ref{PSSVR-2}) onto Grassmann manifolds. Given a set of Grassmann points $\mathcal{X}=\{\mathbf X_1, \mathbf X_2,...,\mathbf X_m\}$ where $\mathbf X_i \in \mathcal{G}(p,d)$ and $m$ is the number of samples. Intuitively translating the PSSVR model to the non-flat Grassmann manifolds results in the following formula:
\begin{equation}
\label{LRRM}
\min_{\mathbf Z}\|\mathbf Z\|_{>r}+\lambda\sum^m_{i=1}\bigg\| \mathbf X_i \ominus (\biguplus^m_{j=1} z_{ij}\odot \mathbf X_j)\bigg\|_{\mathcal{G}},
\end{equation}
where $\left\|\mathbf X_i \ominus (\biguplus^m_{j=1} z_{ij}\odot \mathbf X_j)\right\|_{\mathcal{G}} $ with the operator $\ominus$ represents the manifold distance between $\mathbf X_i$ and its reconstruction $\biguplus^m_{j=1}z_{ij}\odot \mathbf X_j$ denoting the ``\emph{combination}'' operation on the manifold. All these operators $\ominus$, $\odot$ and $\biguplus$ are abstract operations, which represent the `linear operations' to be defined on this manifold. So to establish the PSSVR model on Grassmann Manifolds, one should define a proper distance and proper combination operations on the manifold.

From the geometric property of Grassmann manifolds, we can use the metric on Grassmann manifolds induced by the distance defined in (\ref{Dist_Grassmann}) to replace the manifold distance in (\ref{LRRM}), i.e.
\begin{equation}
\label{new8}
\left\| \mathbf X_i \ominus (\biguplus^m_{j=1}z_{ij}\odot \mathbf X_j)\right\|_{\mathcal{G}} =d_g( \mathbf X_i ,\biguplus^m_{j=1} z_{ij}\odot \mathbf X_j).
\end{equation}
Additionally, from the mapping in (\ref{Grassmann2Sym_mapping}), the mapped points are symmetric matrices 
in $\text{Sym}(d)$, so they have the natural linear combination operation like that in Euclidean space. Thus we can replace the Grassmann points with its mapped points to implement the combination in (\ref{LRRM}), i.e.
\begin{equation}
\label{new9}
\biguplus^m_{j=1}z_{ij}\odot \mathbf X_j =\mathcal{X}\times_3 Z,
\end{equation}
where $ \mathcal{X} = \{ \mathbf X_1\mathbf X_1^T, \mathbf X_2\mathbf X_2^T, ..., \mathbf X_m\mathbf X_m^T\}\subset \text{Sym}(d)$ is a $3$-order tensor which stacks all mapped symmetric matrices along the $3$rd mode, and $\times_3$ is the mode-3 multiplication between a 3-order tensor and a matrix \cite{KoldaBader2009}. Up to now, we can construct the PSSVR model on Grassmann Manifolds as follows,
\begin{equation}
\label{GPLRR}
\min\limits_{\mathcal{E},\mathbf Z} \|\mathbf Z\|_{>r} + \lambda\|\mathcal{E}\|_F^2 \ \ \  \text{s.t.} \ \ \ \mathcal{X}=\mathcal{X}\times_3 \mathbf Z+\mathcal{E},
\end{equation}
where the  reconstructed error $\mathcal E$ is also a $3$-order tensor and the coefficient matrix $\mathbf Z\in \mathbb{R}^{m\times m}$. We call this model the GPSSVR.

\subsection{Algorithm for PSSVR on Grassmann Manifolds}

To solve the GPSSVR problem in (\ref{GPLRR}), we first simplify the representation of the reconstruction tensorial error $\mathcal E$ to avoid the complex calculation between 3-order tensor and a matrix in (\ref{GPLRR}).

Consider the $i$-th front slice $\mathbf E_i$ of the tensor $\mathcal E$, i.e.,
\[
\mathbf E_i = \mathbf X_i\mathbf X_i^T-\sum\limits_{j=1}^{m}z_{ij}(\mathbf X_j\mathbf X_j^T).
\]

Denote
\[
\Delta_{ij} = \text{tr}[(\mathbf X_j^{T}\mathbf X_i)(\mathbf X_i^{T}\mathbf X_j)].
\]

Clearly we have $\Delta_{ij}=\Delta_{ji}$,  hence we can define an $m \times m$ symmetric matrix
\[
\Delta = (\Delta_{ij})_{i=1,j=1}^{m}.
\]

Now it is straightforward to represent the reconstruction error $\|\mathcal E\|_F^2$ as
\begin{equation}
\label{GPLRR_EReconstruct}
\|\mathcal{E}\|_F^2 = \text{tr}(\Delta) -2\text{tr}(\mathbf Z\Delta)+\text{tr}(\mathbf Z\Delta \mathbf Z^T).
\end{equation}

Therefore, substituting (\ref{GPLRR_EReconstruct}) into the objective function in (\ref{GPLRR}) results in an equivalent and solvable optimization model,
\begin{equation}
\label{GPLRR-2}
\min\limits_{\mathbf Z} - 2\lambda\text{tr}(\mathbf Z\Delta) + \lambda\text{tr}(\mathbf Z\Delta \mathbf Z^T) + \|\mathbf Z\|_{>r}.
\end{equation}

To tackle this problem, we use the alternating direction method (ADM) \cite{LinLiuSu2011,BoydParikhChuPeleatoEckstein2011} which is widely used to solve unconstrained convex problems \cite{LiuLinSunYuMa2013,LiuYan2012}. Firstly, we introduce an auxiliary  variable $\mathbf J=\mathbf Z \in \mathbb{R}^{m\times m}$ to separate the terms of variable $\mathbf Z$ and reformulate the optimization problem as follows,
\begin{equation}
\min\limits_{\mathbf Z,\mathbf J} - 2\lambda\text{tr}(\mathbf Z\Delta) + \lambda\text{tr}(\mathbf Z\Delta \mathbf Z^T) + \|\mathbf J\|_{>r} \ \ \text{s.t.} \ \ \mathbf J=\mathbf Z.
\end{equation}

Thus, the ADM method can be applied to absorb the linear constraint into the objective function as follows,
\begin{equation}
\begin{aligned}\label{GPLRR_ALM}
f(\mathbf Z,\mathbf J,\mathbf Y,\mu) &= - 2\lambda\text{tr}(\mathbf Z\Delta) + \lambda\text{tr}(\mathbf Z\Delta \mathbf Z^T)+\|\mathbf J\|_{>r} \\
             &+ \langle \mathbf Y,\mathbf Z-\mathbf J \rangle
             + \frac{\mu}{2}\|\mathbf Z-\mathbf J\|_F^2,
\end{aligned}
\end{equation}
where matrix $\mathbf Y$ is the Lagrangian Multiplier and $\mu$ is a weight to tune the error term $\|\mathbf Z-\mathbf J\|_F^2$.

The ALM formula (\ref{GPLRR_ALM}) can be naturally solved by  alternatively solving for $\mathbf Z$, $\mathbf J$ and $\mathbf Y$, respectively in an iterative procedure. The pseudo code of our proposed method is summarized in ALGORITHM 1. Now, we will analyze how to update these variables in each iteration.

\begin{algorithm}
\renewcommand{\algorithmicrequire}{\textbf{Input:}}
\renewcommand\algorithmicensure {\textbf{Output:} }
\caption{Solving the problem (\ref{GPLRR_ALM}) by ADM.}\label{wholeAlg}
\begin{algorithmic}[1]
\REQUIRE{The Grassmann sample set $\{\mathbf X_i\}_{i=1}^m$, $\mathbf X_i\in \mathcal{G}(p,d)$, the expected rank $r$, and the balancing parameters $\lambda$.}
\ENSURE{The GPSSVR representation $\mathbf Z$.}
\STATE Initialize:$\mathbf J=\mathbf Z=0,\mathbf Y=0,\mu=10^{-6},\mu_{max}=10^{10}, \rho=1.9$ and $\varepsilon=10^{-8}$\;
\FOR{i=1:m}
\FOR{j=1:m}
\STATE $\Delta_{ij}\leftarrow \mbox{tr}[(\mathbf X_j^{T}\mathbf X_i)(\mathbf X_i^{T}\mathbf X_j)]$\;
\ENDFOR
\ENDFOR
\WHILE{not converged}
\STATE fix $\mathbf Z$ and update $\mathbf J$ by \\$\mathbf J\leftarrow \min\limits_{\mathbf J}(\|\mathbf J\|_{>r}+\langle \mathbf Y,\mathbf Z-\mathbf J\rangle+\frac{\mu}{2}\|\mathbf Z-\mathbf J\|_F^2)$ \;
\STATE fix $\mathbf J$ and update $\mathbf Z$ by \\$\mathbf Z = (2\lambda\Delta + \mu \mathbf J - \mathbf Y)(2\lambda \Delta + \mu \mathbf I)^{-1}$ \;
\STATE update the multipliers: \\ $\mathbf Y\leftarrow \mathbf Y+\mu(\mathbf Z-\mathbf J)$ \;
\STATE update the parameter $\mu$ by $\mu \leftarrow \min(\rho\mu,\mu_{\mbox{max}})$ \;
\STATE check the convergence condition: \\ {$\|\mathbf Z-\mathbf J\|_\infty <\varepsilon$} \;
\ENDWHILE
\end{algorithmic}
\end{algorithm}

\subsubsection{Updating $\mathbf J$}
To update $\mathbf J$ at the $(k+1)-$th iteration, we fix $\mathbf Z$, $\mathbf Y$ and $\mu$ to their $k$-th iteration values, respectively, and solve the following problem accordingly:

\begin{equation}
\begin{aligned}\label{PGLRR_subproblemJ}
\mathbf J^{k+1} &  = \arg\min\limits_{\mathbf J}f(\mathbf Z^k,\mathbf J,\mathbf Y^k,\mu^k)  \\
&         = \arg\min\limits_{\mathbf J}\|\mathbf J\|_{>r} + \langle \mathbf Y,\mathbf Z-\mathbf J \rangle + \frac{\mu}{2}\|\mathbf Z-\mathbf J\|_F^2\\
&         = \arg\min\limits_{\mathbf J}(\|\mathbf J\|_{>r}+\frac{\mu}{2}\|\mathbf J-(\mathbf Z+\frac{\mathbf Y}{\mu})\|_F^2)
\end{aligned}
\end{equation}

For the above problem (\ref{PGLRR_subproblemJ}), a closed-form solution is suggested in \cite{OhTaiBazinKimKweon2015} as the following theorem.

\begin{thm}Given that $\mathbf U \mathbf D \mathbf V^T = \text{SVD}(\mathbf Z+\frac{\mathbf Y}{\mu})$ as defined above, the solution to (\ref{PGLRR_subproblemJ}) is given by
\[
\mathbf J^* = \mathbf U (\mathbf D_r + \mathcal{S}_{\mu^{-1}}\mathbf D_{r^{'}}) \mathbf V^T,
\]
where $\mathbf D_r$ and $\mathbf D_{r^{'}}$ are diagonal matrices. $\text{diag}(\mathbf D_r)$ is the $r$ largest singular values and $\text{diag}(\mathbf D_{r^{'}})$ collects all the rest singular values from SVD. The singular value thresholding operator is defined as $\mathcal{S}_\tau[x] = \text{sign}(x)\cdot \max(|x|-\tau,0)$
\end{thm}
\begin{proof}
Please refer to the proof of Lemma 1 in \cite{OhTaiBazinKimKweon2015}.
\end{proof}

\subsubsection{Updating $\mathbf Z$}
To update $\mathbf Z$ at the $(k+1)-$th iteration, we derive the ALM formulation (\ref{GPLRR_ALM}) with fixed $\mathbf J$, $\mathbf Y$ and $\mu$ and obtain the following form:
\begin{equation}
\begin{aligned}\label{PGLRR_subproblemZ}
\mathbf Z^{k+1} &= \arg\min\limits_{\mathbf Z}f(\mathbf Z,\mathbf J^{k+1},\mathbf Y^k,\mu^k)\\
        &= \arg\min\limits_{\mathbf Z}- 2\lambda\text{tr}(\mathbf Z\Delta) + \lambda\text{tr}(\mathbf Z\Delta \mathbf Z^T) + \langle \mathbf Y^k,\mathbf Z-\mathbf J^{k+1} \rangle \\
        &+ \frac{\mu^k}{2}\|\mathbf Z-\mathbf J^{k+1}\|_F^2.
\end{aligned}
\end{equation}

This is a quadratic optimization problem about $\mathbf Z$. The closed-form solution is given by
\begin{equation}
\mathbf Z^{k+1} = (2\lambda\Delta + \mu^k \mathbf J^{k+1} - \mathbf Y^k)(2\lambda\Delta + \mu^k \mathbf I)^{-1}.
\end{equation}

\subsubsection{Updating $\mathbf Y$}
Matrix $\mathbf Y$ is the Lagrangian Multiplier for the linear constraint. Once we have solved the two subproblems about $\mathbf J$ and $\mathbf Z$ respectively in $(k+1)$-th iteration, we can update $\mathbf Y$ easily by the following rule:
\begin{equation}\label{updatingY}
\mathbf Y^{k+1} = \mathbf Y^{k}+\mu^{k}(\mathbf Z^{k+1}-\mathbf J^{k+1}).
\end{equation}

\subsubsection{Adapting Penalty Parameter $\mu$}
For the penalty parameter $\mu > 0$, we could update it by:
\[
\mu^{k+1} = \min(\rho\mu^k,\mu_{\max}),
\]
where $\mu_{\max}$ is the pre-defined upper bound of ${\mu^k}$.

\subsubsection{Termination and Clustering}
After obtaining the GPSSVR representation $\mathbf Z^*$ by ALGORITHM 1, an affinity matrix can be constructed $\mathbf W=\frac{|\mathbf Z^*|+|\mathbf Z^{*}|^T}{2}$. 
Then, this affinity matrix can be used in a spectral clustering algorithm to get the final clustering. As a widely used spectral clustering algorithm for subspace segmentation problems, NCut is chosen in this paper. The whole clustering procedure of the proposed method is summarized in ALGORITHM 2.

\begin{algorithm}
\renewcommand{\algorithmicrequire}{\textbf{Input:}}
\renewcommand\algorithmicensure {\textbf{Output:} }
\caption{Clustering algorithm by the PSSVR on Grassmann Manifolds.}\label{wholeAlg}
\begin{algorithmic}[1]
\REQUIRE{The videos for clustering $\mathcal{X}$.}
\ENSURE{The clustering results of $\mathcal{X}$.}
\STATE Representing $\mathcal{X}$ as a set of Grassmann points;
\STATE Mapping Grassmann points into symmetric space as (\ref{Grassmann2Sym_mapping});
\STATE Obtaining the GPSSVR representation $\mathbf Z^*$ of $\mathcal{X}$ by ALGORITHM 1;
\STATE Computing the affinity matrix $\mathbf W=\frac{|\mathbf Z^*|+|\mathbf Z^{*}|^T}{2}$;
\STATE Implementing NCut($\mathbf W$) to get the final clustering result of $\mathcal{X}$.
\end{algorithmic}
\end{algorithm}

\subsection{Computational Complexity}

The computational complexity of ALGORITHM 1 could be divided into two parts: the data representation and the solution to the problem.

In the data representation part, $\Delta$ is calculated by using the trace operation, therefore, for the $m$ samples, the computational complexity of calculating $\Delta$ should be $\mathcal{O}(m^2)$. In the second part of the algorithm, the major computational complexity is the SVD decomposition of an $m\times m$ matrix for updating $\mathbf J$, costing $\mathcal{O}(m^3)$ computational time. However there is no need to calculate all the singular values due to the thresholding operation, instead calculating up to for example $4r$ first singular values by using the partial SVD like \cite{LiuLinYu2010}.
Thus, for the $s$ iterations, the total cost of calculating the solution to the model is $\mathcal{O}(srm^2)$.
The overall computational complexity is $\mathcal{O}(m^2)+\mathcal{O}(srm^2)$.

\section{Laplacian PSSVR on Grassmann Manifolds (LapGPSSVR)}\label{Sec:4}

\subsection{Laplacian PSSVR on Grassmann Manifolds}
For the self-representation based methods, the column of $\mathbf Z$, denoted by $\textbf z_i$, can be regarded as a new representation of data $\textbf x_i$, and $z_{ij}$ represents the similarity between data $\textbf x_i$ and $\textbf x_j$, accordingly. In the proposed method GPSSVR (\ref{GPLRR}), the global structure is enforced by the global constraint of rank minimization of the matrix $\mathbf Z$. To incorporate more local similarity information into $\mathbf Z$ in our model, we consider imposing the local geometrical structures. For this purpose, Laplacian matrix regularization is naturally regarded as a proper choice because it can maintain similarity between data. Thus a Laplacian Partial Sum Minimization of Singular Values Representation on Grassmann Manifolds model, termed LapGRSSVR, can be formulated as
\begin{equation}
\label{LapGPlrr-1}
\min\limits_{\mathbf Z,\mathcal E}\|\mathbf Z\|_{>r} + \lambda\|\mathcal E\|_F^2 + \beta\sum\limits_{i,j}\| \textbf z_i- \textbf z_j\|_2^2 w_{ij}, \ \ \text{s.t.} \ \ \mathcal{X} = \mathcal{X}{\times_3}\mathbf Z + \mathcal{E}.
\end{equation}
where $w_{ij}$ denotes the local similarity between Grassmann points $\mathbf X_i$ and $\mathbf X_j$. There are many ways to define $w_{ij}$'s. In this paper, we simply use the explicit neighborhood determined by its manifold distance measure to define all the $w_{ij}$. Let $C$ be a parameter of neighborhood size, and we define $w_{ij} = d_g(\mathbf X_i,\mathbf X_j)$ if $\mathbf X_i \in \mathcal{N}_C(\mathbf X_j)$; otherwise $w_{ij}=0$, where $\mathcal{N}_C(\mathbf X_j)$ denotes the $C$ nearest elements of $\mathbf X_j$ on Grassmann manifolds.

By introducing the Laplacian matrix $\mathbf L$, problem (\ref{LapGPlrr-1}) can be easily re-written as its Laplacian form,

\begin{equation}
\label{LapGPSSVR-2}
\min\limits_{\mathcal E,\mathbf Z}\|\mathbf Z\|_{>r} + \lambda\|\mathcal E\|_F^2 + 2\beta \text{tr}(\mathbf Z\mathbf L\mathbf Z^T) \ \ \text{s.t.} \ \ \mathcal{X} = \mathcal{X}_{\times_3}\mathbf Z + \mathcal E,
\end{equation}
where the Laplacian matrix $\mathbf L \in R^{m\times m}$ is defined as $\mathbf L=\mathbf D-\mathbf W$, and $\mathbf W=[w_{ij}]_{i=1,j=1}^{m}$ and $\mathbf D = \text{diag}(d_{ii})$ with $d_{ii}=\sum\limits_j w_{ij}$.

\subsection{Algorithm for Laplacian PSSVR on Grassmann Manifolds}
Similar to deriving the algorithm for GPSSVR, problem (\ref{LapGPSSVR-2}) can be easily converted to the following model:
\begin{equation}
\label{LapGPSSVR-3}
\min\limits_{\mathbf Z} - 2\lambda\text{tr}(\mathbf Z\Delta) + \lambda\text{tr}(\mathbf Z\Delta \mathbf Z^T) + 2\beta\text{tr}(\mathbf Z\mathbf L\mathbf Z^T) + \|\mathbf Z\|_{>r}.
\end{equation}

Thus the ADM \cite{LinLiuSu2011} can also be employed to solve this problem. Letting $\mathbf J=\mathbf Z$ to separate the variable $\mathbf Z$ from the terms in the objective function, we can formulate the following problem for (\ref{LapGPSSVR-3}),

\begin{equation}\label{LapPSSVR_ALM}
\min\limits_{\mathbf Z,\mathbf J} - 2\lambda\text{tr}(\mathbf Z\Delta) + \lambda\text{tr}(\mathbf Z\Delta \mathbf Z^T) + 2\beta\text{tr}(\mathbf Z\mathbf L\mathbf Z^T) + \|\mathbf J\|_{>r}, \ \text{s.t.} \ \mathbf J=\mathbf Z.
\end{equation}

So, its ALM formulation can be defined as the following unconstrained optimization,
\begin{equation}
\begin{aligned}\label{Lap}
f(\mathbf Z,\mathbf J,\mathbf Y,\mu) &= - 2\lambda\text{tr}(\mathbf Z\Delta) + \lambda\text{tr}(\mathbf Z\Delta \mathbf Z^T) + 2\beta\text{tr}(\mathbf Z\mathbf L\mathbf Z^T) \\
&+ \|\mathbf J\|_{>r} + \langle \mathbf Y,\mathbf Z-\mathbf J \rangle + \frac{\mu}{2}\|\mathbf Z-\mathbf J\|_F^2.
\end{aligned}
\end{equation}

This problem can be solved by solving the two subproblems (\ref{LapPSSVR_subproblemJ}) and (\ref{LapPSSVR_subproblemZ}) below,
\begin{equation}
\begin{aligned}\label{LapPSSVR_subproblemJ}
\mathbf J^{k+1}  &= \arg\min\limits_{J}f(\mathbf Z^k,\mathbf J,\mathbf Y^k,\mu^k) \\
         &= \arg\min\limits_{\mathbf J}\|\mathbf J\|_{>r} + \langle \mathbf Y^k,\mathbf Z^k-\mathbf J \rangle + \frac{\mu^k}{2}\|\mathbf Z^k-\mathbf J\|_F^2, \\
         &= \arg\min\limits_{\mathbf J}(\|\mathbf J\|_{>r}+\frac{\mu^k}{2}\|\mathbf J-(\mathbf Z^k+\frac{\mathbf Y^k}{\mu^k})\|_F^2)
\end{aligned}
\end{equation}
and
\begin{equation}
\begin{aligned}
\label{LapPSSVR_subproblemZ}
\mathbf Z^{k+1} &= \arg\min_{\mathbf Z}f(\mathbf Z,\mathbf J^{k+1},\mathbf Y^k,u^k) \\
        &= \arg\min\limits_{\mathbf Z}- 2\lambda\text{tr}(\mathbf Z\Delta) + \lambda\text{tr}(\mathbf Z\Delta \mathbf Z^T) + 2\beta \text{tr}(\mathbf Z\mathbf L\mathbf Z^T). \\
        &+ \langle \mathbf Y^k,\mathbf Z-\mathbf J^{k+1} \rangle + \frac{\mu^k}{2}\|\mathbf Z-\mathbf J^{k+1}\|_F^2.
\end{aligned}
\end{equation}

Both (\ref{LapPSSVR_subproblemJ}) and (\ref{LapPSSVR_subproblemZ}) can be solved similar to (\ref{PGLRR_subproblemJ}) and (\ref{PGLRR_subproblemZ}), respectively. For example, the solution to (\ref{LapPSSVR_subproblemZ}) is given by
\begin{equation}
\mathbf Z^{k+1} = (2\lambda\Delta + \mu^k \mathbf J^{k+1} - \mathbf Y^k)(2\lambda\Delta + 2\beta \mathbf L + \mu^k \mathbf I)^{-1}.
\end{equation}

\subsection{Convergence Analysis}
Lin \emph{et al}. \cite{LinLiuSu2011} proposed a linearized ADMM (LADMM) method, in which the linearization is performed over the augmented quadratic penalty term from the linear constrain conditions and the algorithm convergence is theoretically guaranteed. However the algorithm convergence analysis in \cite{LinLiuSu2011} cannot be applied to ALGORITHM 1 in Section \ref{Sec:3} as well as the algorithm for LapGPSSVR in Section \ref{Sec:4} due to the non-convexity of the PSSV norm in the objective functions. However the authors of \cite{OhTaiBazinKimKweon2015} propose a convergence analysis for the PSSV method based on the Lipschitz property of the PSSV norm, see the supplementary material of \cite{OhTaiBazinKimKweon2015}.

The objective functions (\ref{GPLRR_ALM}) and (\ref{Lap}) are different from the objective function in  \cite{OhTaiBazinKimKweon2015} because of those trace terms in second order of $\mathbf Z$. Fortunately splitting variable by setting $\mathbf Z=\mathbf J$ results in an algorithm in which there is no linearization procedure for $\mathbf J$ while the subproblem for $\mathbf Z$ has a closed-form solution. Hence all the analysis in \cite{OhTaiBazinKimKweon2015} is valid for our algorithms in this paper. The convergence proof can be formed in the same way as that in \cite{OhTaiBazinKimKweon2015}. For example, we can work out the KKT conditions for both (\ref{GPLRR_ALM}) and (\ref{Lap}), based on the generalized sub-gradient of the PSSV norm. And finally we can have

\begin{thm}\label{Convergencethm}[Convergence] Let $S^k=(\mathbf Z^k,\mathbf J^k,\mathbf Y^k, \hat{\mathbf Y}^k)$ where $\hat{\mathbf Y}^{k+1}=\mathbf Y^k+\mu^k(\mathbf Z^k-\mathbf J^{k+1})$. If $\{\mathbf Y^k\}_{k=1}^{\infty}$ and $\{\hat{\mathbf Y}^{k}\}_{k=1}^{\infty}$ are bounded, $\lim\limits_{k\rightarrow\infty}(\mathbf Y^{k+1}-\mathbf Y^k)=0$, and $\mu^k$ is non-decreasing, then any accumulation point of $\{S^k\}_{k=1}^{\infty}$ satisfies the KKT condition. In particular, whenever $\{S^k\}_{k=1}^{\infty}$ converges, it converges to a KKT point of problem (\ref{GPLRR_ALM}) (or (\ref{Lap})).
\end{thm}

For better flow of the paper, we move the proof of Theorem \ref{Convergencethm} to Appendix. Theorem \ref{Convergencethm} guarantees a converged solution if the sequence produced by the Algorithms is convergent. In our proposed algorithms, each subproblem has a closed-form solution and the value tends to be stable along with increasing iteration. In addition, the experimental results (see Fig. \ref{covergencefig}) also demonstrate that our ADM-based algorithms have a strong convergence property. It is still a challenging task, to the best of our knowledge, to prove the general convergence property of ADM in the cases of existence of non-convex components in objective functions. The ADM for non-convex problems can be considered as a local optimization method, which aims to converge to a point with better objective value.

\begin{figure*}
    \begin{center}
    \includegraphics[width=0.85\textwidth]{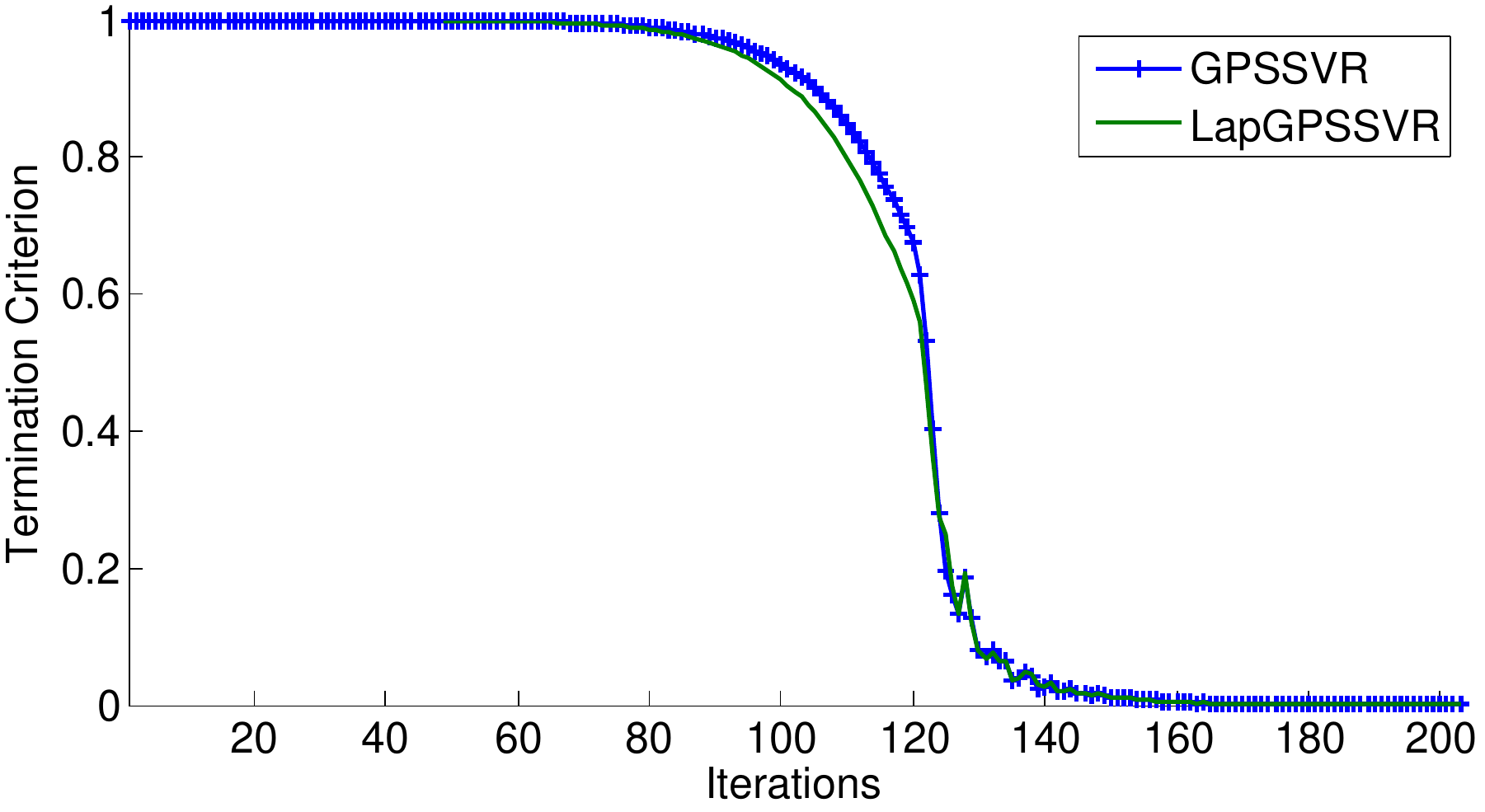}
    \end{center}
    \caption{Convergence behavior of our proposed methods on Extended Yale B dataset.}\label{covergencefig}
\end{figure*}

\section{Experiments}\label{Sec:5}
In this section, to test the effectiveness of our proposed methods, we conduct several unsupervised clustering experiments on image sets and different video datasets. One facial image dataset and the four video datasets used in our experiments are listed below: Extended Yale B face dataset
\footnote{\url{http://vision.ucsd.edu/content/yale-face-database}.},
SKIG action dataset \footnote{\url{http://lshao.staff.shef.ac.uk/data/SheffieldKinectGesture.htm}.},
Ballet video dataset \footnote{\url{https://www.cs.sfu.ca/research/groups/VML/semilatent/}.},
UCF sports dataset \footnote{\url{http://crcv.ucf.edu/data/UCF_Sports_Action.php}.},
Highway traffic dataset \footnote{\url{http://www.svcl.ucsd.edu/projects/traffic/}.}.

To demonstrate the performance of GPSSVR and LapGPSSVR methods, we compare them with several state-of-the-art clustering methods. Since our methods are related to LRR and manifold models, we mainly select LRR based methods or manifold based methods as baselines, which are listed below:
\begin{itemize}
  \item \textbf{Sparse Subspace Clustering (SSC) \cite{ElhamifarVidal2013}}: The SSC model aims to find the sparsest representation for each datum using $\ell_1$ regularization.
  \item \textbf{Low Rank Representation (LRR) \cite{LiuLinSunYuMa2013}}: The LRR model represents the holistic correlation among the data by using the nuclear norm regularization.
  \item \textbf{Low Rank Representation on Grassmann Manifolds (GLRR-F) \cite{WangHuGaoSunYin2014}}: The GLRR-F model embeds the image sets onto Grassmann manifolds and extends the LRR model to the case on the Grassmann manifolds space.
  \item \textbf{Statistical computations on Grassmann and Stiefel manifolds (SCGSM) \cite{TuragaVeeraraghavanSrivastavaChellappa2011}}: The SCGSM model explores statistical modeling methods that are derived from the Riemannian geometry of the manifold.
  \item \textbf{Sparse Manifold Clustering and Embedding (SMCE) \cite{ElhamifarVidalNips2011}}: The SMCE model utilizes the local manifold structure to find a small neighborhood around each data point and connects each point to its neighbors with appropriate weights.
  \item \textbf{Latent Space Sparse Subspace Clustering (LS3C) \cite{PatelNguyenVidal2013}}: The LS3C model describes a method that learns the projection of data and finds the sparse coefficients in the low-dimensional latent space.
\end{itemize}

In all the experiments, the input raw data are image sets derived from video clips. To represent them as Grassmann points, for a video clip with $M$ frames, denoted by $\{\mathbf Y_i\}_{i=1}^{M}$, where $\mathbf Y_i$ is the $i$-th gray frame with dimension $a\times b$,  we construct a matrix $\mathcal{Y}=[\text{vec}(\mathbf Y_1), \text{vec}(\mathbf Y_2), ..., \text{vec}(\mathbf Y_M)]$ of size $(a\times b)\times M$. Thus, a Grassmann point can be generated by any orthogonalization procedure of $\mathcal{Y}$. For convenience, we select SVD decomposition in our experiments i.e. $\mathcal{Y}=\mathbf U\Sigma \mathbf V^T$. Then we pick up the first $p$ singular-vectors of $\mathbf U$ as the representation of a Grassmann point $\mathbf X\in\mathcal{G}(p,a\times b)$. The same way can be applied to any given image set for its Grassmann representation.

To execute all the other comparing methods, we should formulate proper data representation for each of them as different methods demand different types of data inputs for clustering. The baseline subspace clustering methods, LRR and SSC, take as inputs vectorial data. They cannot be applied directly on data in form of Grassmann points. So we have to vectorize each video clip as vectorial inputs. However, a direct vectorization results in very high dimensional vectors which are hard to be handled on a normal PC. Thus we apply PCA to reduce these vectors to a low dimension which equals to the number of PCA components retaining 95\% of its variance energy. The PCA projected vectors are taken as inputs for both SSC and LRR.

For the manifold related methods, SCGSM clustering can be directly implemented on our Grassmann representation $\mathbf X\in \mathcal{G}(p,a\times b)$ for videos/image sets. Since GPSSVR, LapGPSSVR, GLRR-F methods all embed Grassmann points into symmetric matrix space, we construct the corresponding symmetric matrix $\mathbf X\mathbf X^T\in \mathbb{R}^{(a\times b)\times (a\times b)}$ as their inputs. Although SMCE and LS3C belong to manifold learning methods, they demand vectors as inputs too. However, vectorizing the Grassmann point $\mathbf X \in \mathcal{G}(p,a\times b)$ will destroy the geometry of data, hence we vectorize the correspondent symmetric matrix $\mathbf X\mathbf X^T$ as their inputs. That is, SMEC and LS3C take $vec(\mathbf X\mathbf X^T)$ as inputs for clustering.

To obtain good experimental performances, the model parameters $\lambda$, $\beta$, $r$ and $C$ should be assigned properly. First of all, we should give a better estimate to the expected rank $r$. For a Grassmann point $\mathbf X \in \mathcal{G}(p,d)$, the rank of its matrix representative $\mathbf X$, even the mapped symmetric matrix $\mathbf X\mathbf X^T$ (referring to (\ref{Grassmann2Sym_mapping})), are equal to $p$.
Considering that a mapped symmetric matrix is linearly represented by other mapped symmetric matrices, so we may expect that the expected rank $r$ of the coefficient matrix $\mathbf Z$ is relatively not larger than $p$. According to this analysis, in experiments, we tune the expected rank $r$, e.g. around $4$, to acquire the best experiment results.

$\lambda$ and $\beta$ are the most important penalty parameters for balancing the reconstructed error term, low-rank term and Laplacian regularization term. Empirically, the best value of $\lambda$ depends on the application problems and has a large range for different applications from $0.01$ to $20$. While the value of $\beta$ is usually very small, ranging from $1.0\times 10^{-4}$ to $1.0\times 10^{-2}$, because the value of Laplacian regularization term is usually many times larger than two other terms. As for $C$, which defines the neighborhood size of each Grassmann point, lots of experimental results suggest that a value slightly larger than the average value of the numbers of data in different clusters is a good choice. We tune the models for $C$ around this initial value for different applications.

In our experiments, the performance of different algorithms is measured by the following  clustering accuracy
\[
\text{Accuracy} = \frac{\text{number of correctly classified points}}{\text{total number of points}}\times 100\%.
\]

All the algorithms are coded in Matlab 2014a and implemented on an Intel Core i7-4770K 3.5GHz CPU machine with 32G RAM.

\subsection{Clustering on Face Image sets}

Face clustering is one of the hottest topics in computer vision and machine learning. Affected by various factors, i.e., expression, illumination conditions and light directions, clustering based on individual faces does not achieve great experimental performance. Therefore, we test our proposed methods for the purpose of clustering face image sets, where each face image set contains several numbers of face images of one subject.

The extended Yale B dataset is captured from $38$ subjects and each subject has $64$ front face images in different light directions and illumination conditions. All images are resized into $20\times 20$ pixels. Some face samples in extended Yale B dataset are shown in Fig. \ref{yalefig}.

To generate the experimental data, we form each facial image set by randomly choosing $M=8$ images from the same subject. We set the dimension of subspace as $p=4$ for each Grassmann point. Thus, the Grassmann point in this experiment can be denoted as $\mathbf X\in \mathcal{G}(4,20\times 20)$. The expected rank and the neighbor size are set as $r=4$ and $C=7$, respectively. And parameters $\lambda$ and $\beta$ are set as $0.59$ and $0.001$. As for baselines SSC and LRR, the vector dimension of an image set in size of $20\times 20\times 8=3200$ are reduced to $146$ by PCA.

All experimental results are shown in Table \ref{yaletab}. Compared with the manifold-based methods, i.e., SCGSM, SMCE and LS3C, the excellent performance of our proposed methods demonstrates the low rank constraint on similarity matrix $\mathbf Z$ plays an active role. Our proposed methods are also a little superior to GLRR-F, which verifies PSSV norm improves our proposed methods. The fact that the performance of SSC and LRR is greatly worse than the mentioned manifold-based baselines proves that incorporating manifold geometry is useful for clustering algorithms. Fig. \ref{rankfig} 
clearly shows that a slightly larger expected rank value may help improve the clustering accuracy.

\begin{figure*}
    \begin{center}
    \includegraphics[width=0.85\textwidth]{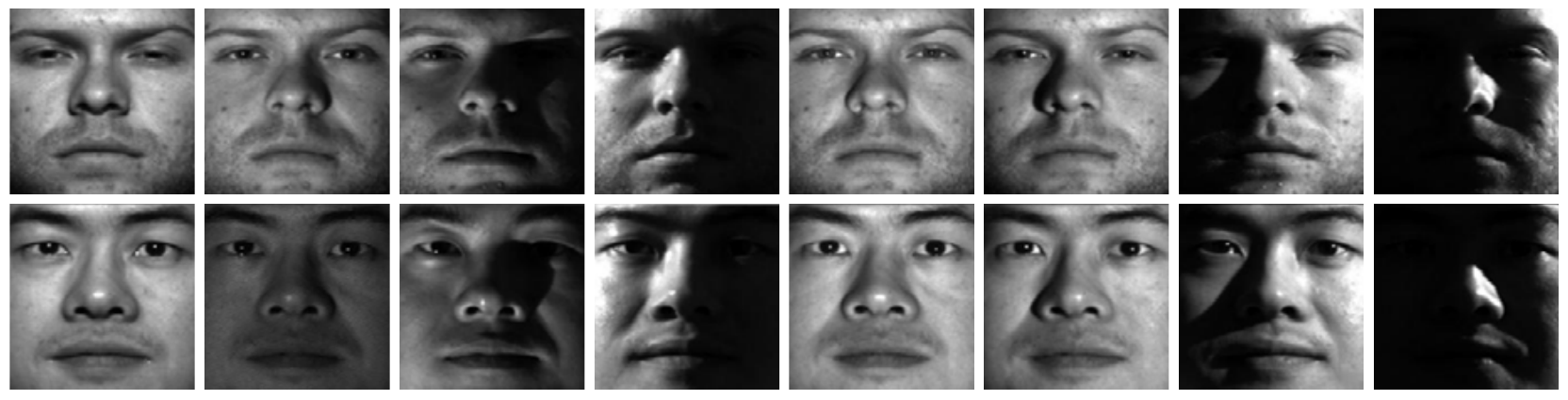}
    \end{center}
    \caption{Some samples from the extended Yale B dataset. Each row denotes an image set sample which contains $8$ face images captured from different light directions and illuminations.}\label{yalefig}
\end{figure*}

\begin{table*}
   \centering
   \caption{Subspace clustering results on the Extended Yale B dataset.\label{yaletab}}{
   \begin{tabular}{|c|c|c|c|c|c|c|c|c|}
     \hline
             Methods & GPSSVR & LapGPSSVR & GLRR-F & LRR & SSC & SCGSM & SMCE & LS3C \\
             \hline
             Yale B & \textbf{0.9024} & \textbf{0.9226} & 0.8878 & 0.2788 & 0.3109 & 0.5657 & 0.8429 & 0.6250 \\
     \hline
   \end{tabular}}
\end{table*}

\begin{figure*}
    \begin{center}
    \includegraphics[width=0.6\textwidth]{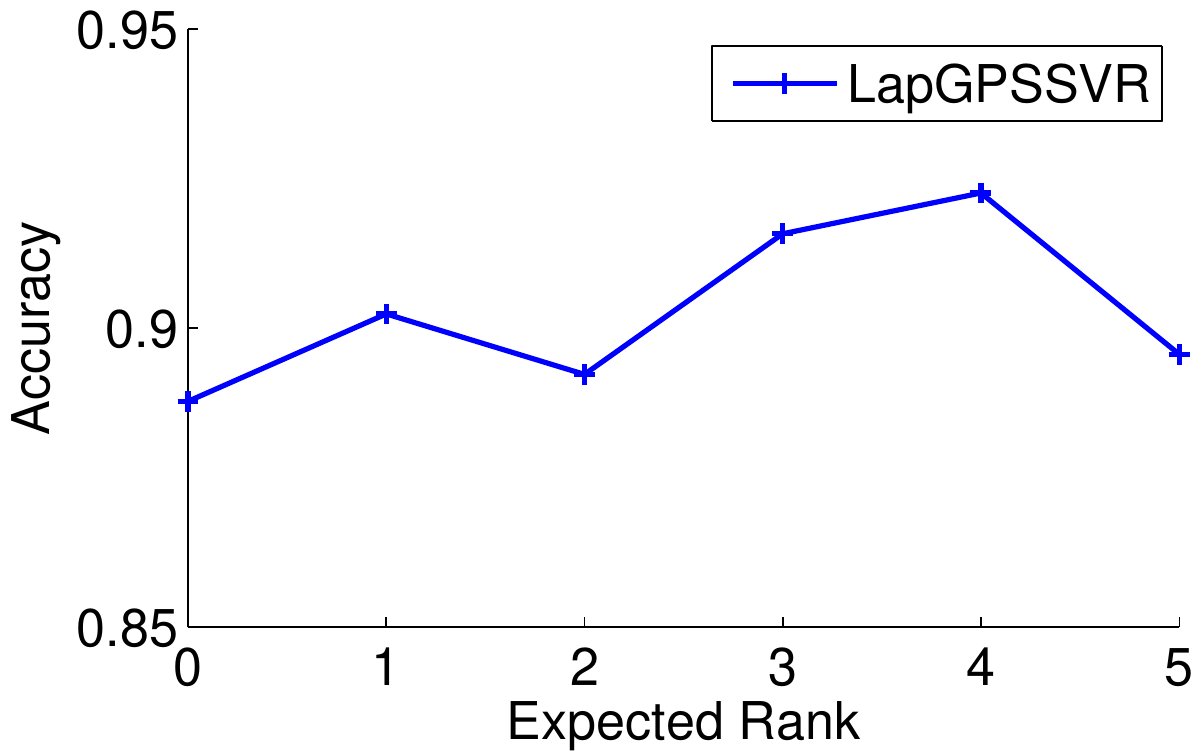}
    \end{center}
    \caption{Comparing the clustering performance of our method for the expected rank from 0 to 5 on Yale B dataset.}\label{rankfig}
\end{figure*}

\subsection{Clustering on Human Action}
GPSSVR model maintains the first $r$ singular values unconstrained to preserve the most dominant information as much as possible. At the same time, it minimizes the sum of the rest of singular values to seek a low rank global structure. Thus the proposed methods are generally more suitable for video and image sets clustering.

In the next experiment on human actions, we select two challenge action video datasets,  Ballet dataset and UCF sport dataset, to test the performance of the proposed methods. With simple backgrounds, the Ballet dataset is an appropriate benchmark choice to verify the capacity of the proposed method for action recognition in a rather ideal condition; while the UCF sport dataset containing more variations on scenario and viewpoint can be used to examine the robustness of the proposed methods in noised scenarios.

\subsubsection{Ballet Action Dataset}
This dataset \cite{FathiMori2008} contains 44 video clips, collected from an instructional ballet DVD. The dataset consists of 8 complex action patterns performed by 3 subjects. The eight actions include: `left-to-right hand opening', `right-to-left hand opening', `standing hand opening', `leg swinging', `jumping', `turning', `hopping' and `standing still'. The dataset is challenging due to the significant intra-class variations in terms of speed, spatial and temporal scale, cloth texture and movement. The frame images are normalized and centered in a fixed size of $30 \times 30$. Some frame samples of Ballet dataset are shown in Fig.~\ref{FigE4}.
\begin{figure*}
    \begin{center}
    \includegraphics[width=0.85\textwidth]{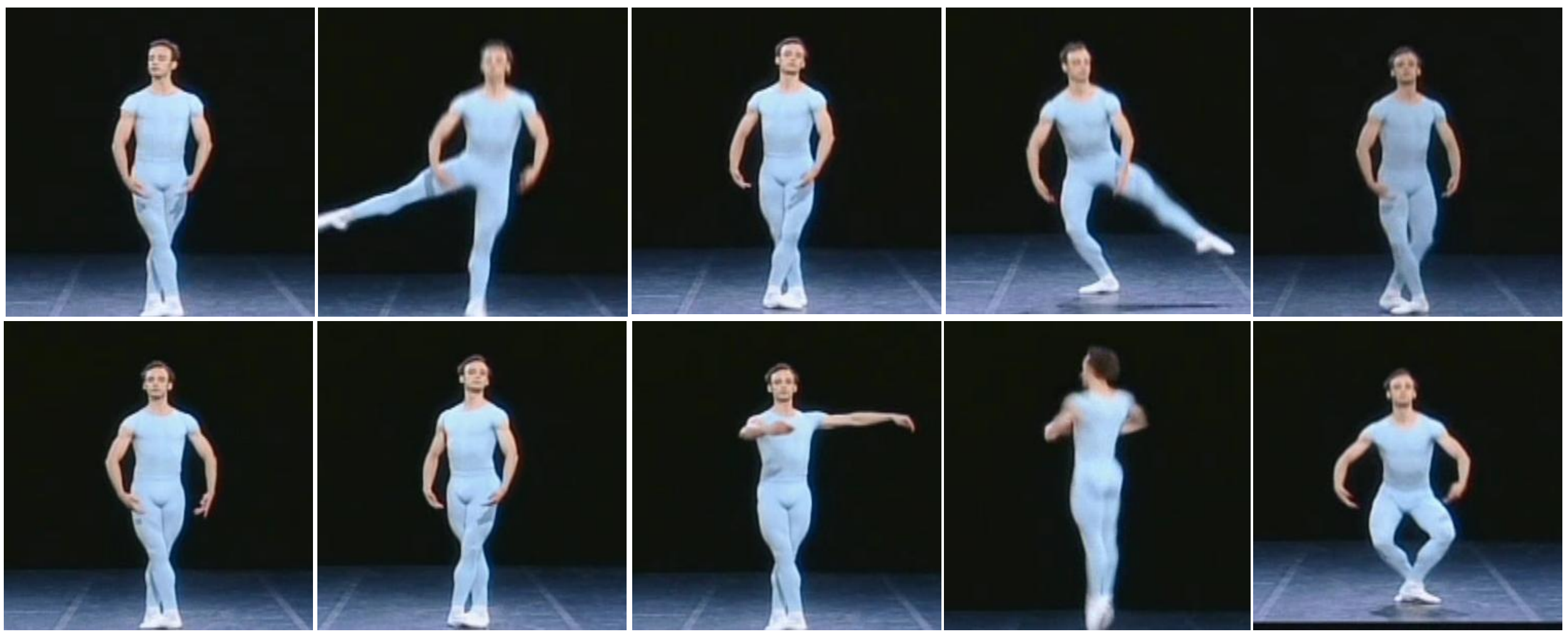}
    \end{center}
    \caption{Some samples from the Ballet Action dataset. Each row denotes a kind of action.}\label{FigE4}
\end{figure*}

We split each clip into subgroups of $M=12$ images and each subgroup is treated as an individual image set. As a result, we construct a total of $713$ image sets which are labeled as $8$ clusters. The dimension of subspace is set to $p=6$ and the Grassmann point can be represented as $\mathbf X_i \in \mathcal{G}(6,900)$. For the setting of rank $r$, we test a number of different values from $1$ to $6$ and find the best expected rank $r$ is 3 in this experiment. The neighborhood size $C$ is tuned to $90$ according to the experimental results. We set the parameters $\lambda=0.03$ and $\beta=0.003$. 
 For LRR and SSC methods, the dimension of vectors in the subspace $30\times 30 \times 12 = 10800$ is reduced to $135$ by PCA.

Table \ref{Ballettab} presents the experimental results of all the algorithms on the Ballet dataset. Although this dataset contains no complex background or illumination changes and can be regarded as a clean human action data in ideal condition, the accuracy in Table \ref{Ballettab} is not very high for the case of eight clusters. The reason is that some types of actions are too similar to each other, e.g., 'left-to-right hand opening' and 'right-to-left hand opening'. Compared with GLRR-F method which is based on the nuclear norm regularization, the proposed methods give a higher clustering accuracy. This demonstrates the benefits of minimizing partial sum of smaller singular values and leaving the $r$ largest singular values unconstrained to preserve as much discrimination information as possible.  Of course, our proposed methods are also obviously superior to other classical methods.

\begin{table*}
   \centering
   \caption{Subspace clustering results on the Ballet dataset.\label{Ballettab}}{
   \begin{tabular}{|c|c|c|c|c|c|c|c|c|}
     \hline
             Methods & GPSSVR & LapGPSSVR & GLRR-F & LRR & SSC & SCGSM & SMCE & LS3C \\
             \hline
             Ballet & \textbf{0.6059} & \textbf{0.6255} & 0.5905 & 0.2819 & 0.2903 & 0.5877 & 0.5105 & 0.4222 \\
     \hline
   \end{tabular}}
\end{table*}

\subsubsection{UCF Sports action Dataset}
This dataset \cite{RodriguezAhmedShah2008} consists of a set of actions collected from various sport matches which are typically featured on broadcast television channels. The dataset includes a total of 150 sequences. The collection has a natural pool of actions with a wide range of scenes and viewpoints, so it is difficult for clustering. There are 10 actions in this dataset: `Diving', `Golf Swing', `Kicking', `Lifting', `Riding Horse', `Running', `Skate Boarding', `Swing-Bench', `Swing-Side', and `Walking'. Each sequence has 22 to 144 frames. We convert these video clips into gray images and each image is resized into $30\times 30$.

We regard each video clip as an image set. Note that the number of frames $M$ of each video clip is various for different video clips. There are totally $150$ image sets and $10$ clusters in this experiment. We select $p=10$ as the dimension of subspace for each Grassmann point. Therefore, a Grassmann point can be represented as $\mathbf X_i \in \mathcal{G}(10,900)$. The expected rank $r$ is set to $4$ and the neighbor size $C$ is $12$. The parameters $\lambda$ and $\beta$ are set as $1.2$ and $0.004$, respectively. The PCA algorithm requires the image sets with the same number of samples, but the RGB sequences contains various frames from 22 to 144. Throwing away too many frames by averaging sampling for longer sequences in the PCA algorithm is unfair for LRR and SSC methods, so we have to give up comparing with LRR and SSC in this experiment.

The experimental results are reported in Table \ref{UCFtab}. Although this challenging dataset has complex backgrounds, viewpoints changes and scales variations, the accuracy result seems be higher than the Ballet dataset. We conclude that the bigger movement in  sport actions helps to distinguish action clusters, resulting in higher accuracy.

\begin{figure*}
    \begin{center}
    \includegraphics[width=0.85\textwidth]{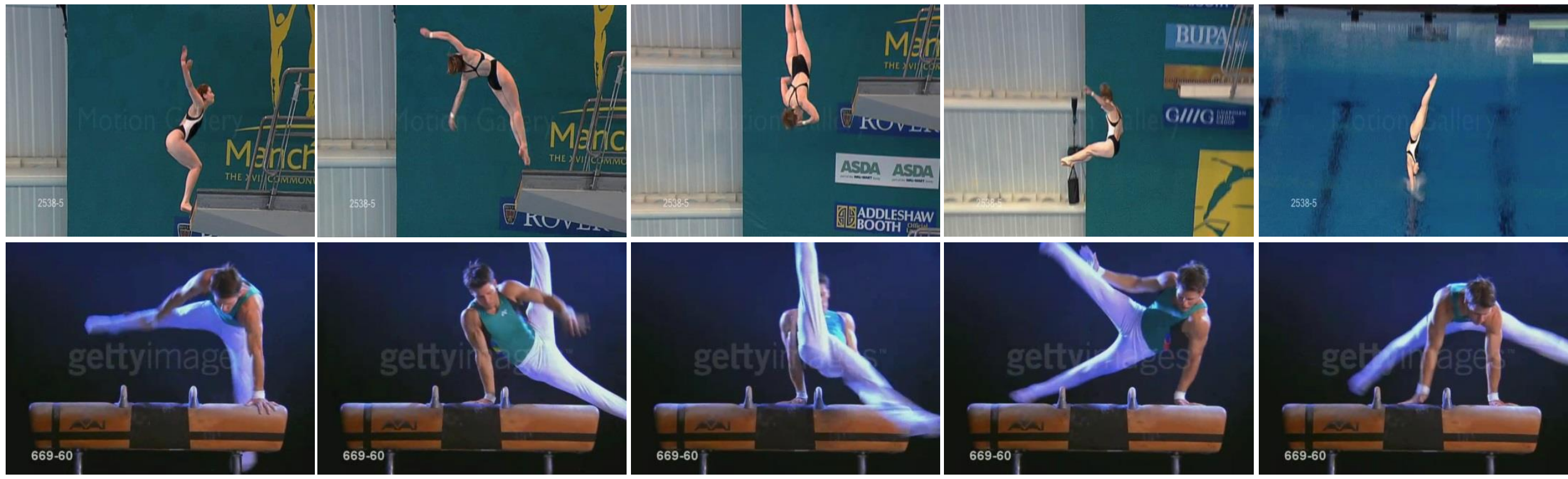}
    \end{center}
    \caption{Some samples from the UCF sports dataset and each row represents a kind of action.}\label{FigE6}
\end{figure*}

\begin{table*}
   \centering
   \caption{Subspace clustering results on the UCF sport dataset.\label{UCFtab}}{
   \begin{tabular}{|c|c|c|c|c|c|c|c|c|}
     \hline
             Methods & GPSSVR & LapGPSSVR & GLRR-F &  SCGSM & SMCE & LS3C \\
             \hline
             UCF & \textbf{0.6800} & \textbf{0.6933} & 0.6533 & 0.5333 & 0.6200 & 0.4667 \\
     \hline
   \end{tabular}}
\end{table*}

\subsection{Clustering on Gesture Action}
The SKIG dataset \cite{LiuShao2013} contains 1080 RGB-D sequences captured by a Kinect sensor. In this dataset, there are ten kinds of gestures of six persons: `circle', `triangle', `up-down', `right-left', `wave', `Z', `cross', `comehere', `turn-around', and `pat'. All the gestures are performed by fist, finger and elbow, respectively, under three backgrounds (wooden board, white plain paper and paper with characters) and two illuminations (strong light and poor light). Each RGB-D sequence contains a set of frames (63 to 605). Here the images are normalized to $24\times 32$ with mean zero and unit variance. Fig. \ref{FigE3} shows some samples of RGB images. In our experiments, we only use the RGB sequences in SKIG dataset.
\begin{figure*}
    \begin{center}
    \includegraphics[width=0.85\textwidth]{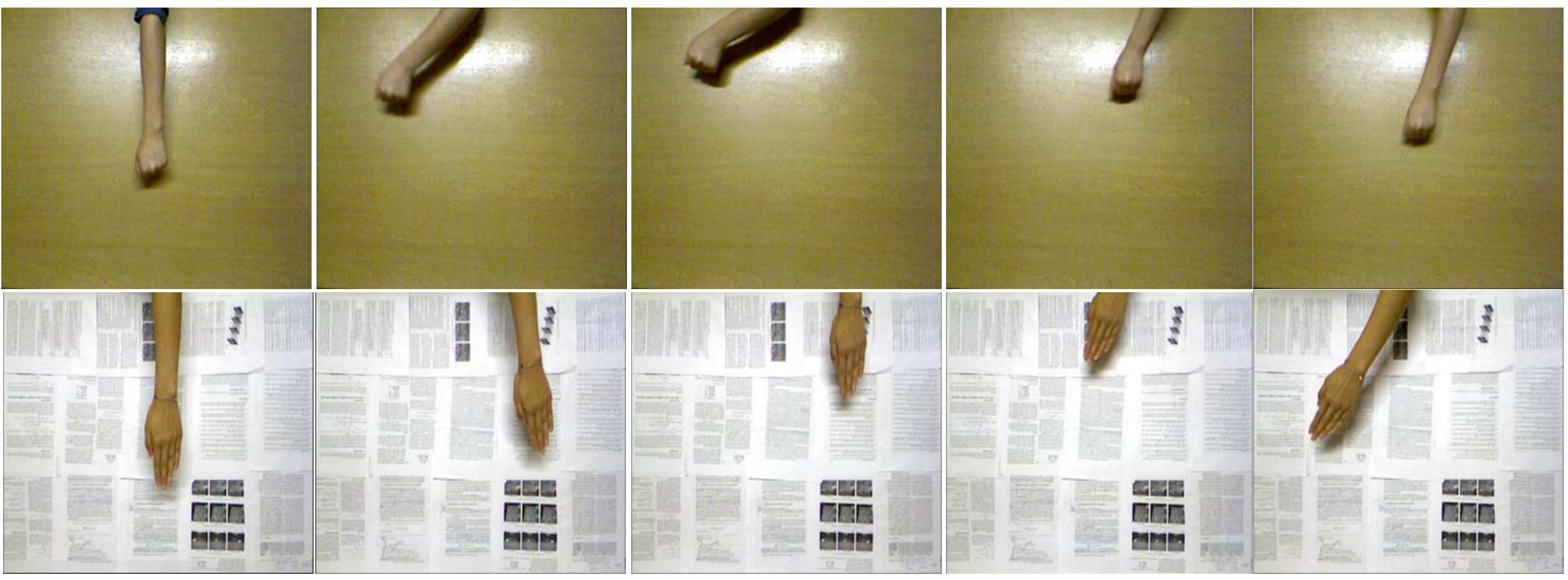}
    \end{center}
    \caption{Some samples from the SKIG Action dataset and each row shows a kind of gesture.}\label{FigE3}
\end{figure*}

Similar to the previous experiments, each video clip is considered as an image set, thus a total of $540$ image sets are labeled as $10$ clusters. We preserve $p=10$ as the dimension of subspaces, so the Grassmann point is represented as $\mathbf X_i \in \mathcal{G}(10,768)$. In this experiment, we empirically set the expected rank and the neighbor size to $r=1$ and $C=65$, respectively. And $\lambda=0.8$ and $\beta=0.009$ are chosen as experimental parameters. We did not conduct experiments for LRR and SSC due to the similar reason mentioned in the last experiment.

Table \ref{SKIGtab} presents all the experimental results on SKIG dataset. Compared with human action datasets, the movement in this gesture dataset is smaller, and the illumination and background are more variate, therefore clustering task on this dataset is more challenging. Our proposed methods, especially LapGPSSVR method, have improved clustering accuracy over all other methods. Except for the discrimination information coming from the first $r$ largest singular values,  the Laplacian regularization also provides meaningful information for clustering.

\begin{table*}
\centering
   \caption{Subspace clustering results on the SKIG dataset.\label{SKIGtab}}{
   \begin{tabular}{|c|c|c|c|c|c|c|c|c|}
     \hline
             Methods & GPSSVR & LapGPSSVR & GLRR-F &  SCGSM & SMCE & LS3C \\
             \hline
             SKIG & \textbf{0.55} & \textbf{0.5981} & 0.5056 & 0.3704 & 0.4611 & 0.4148 \\
     \hline
   \end{tabular}}
\end{table*}

\subsection{Clustering on Natural Scene}

In this experiment, we wish to inspect the proposed methods on practical applications in more complex conditions, such as Traffic Dataset. The traffic dataset \cite{ChanVasconcelos2008} used in this experiment contains 253 video sequences of highway traffic captured under various weather conditions, such as sunny, cloudy and rainy. These sequences are labeled with three traffic levels: light, medium and heavy. There are 44 clips at heavy level, 45 clips at medium level and 164 clips at light level. Each video sequence has 42 to 52 frames. The video sequences are converted to gray images and each image is normalized to size  $24 \times 24$ with mean zero and unit variance. Some samples of the Highway traffic dataset are shown in Fig.~\ref{FigE5}.
\begin{figure*}
    \begin{center}
    \includegraphics[width=0.85\textwidth]{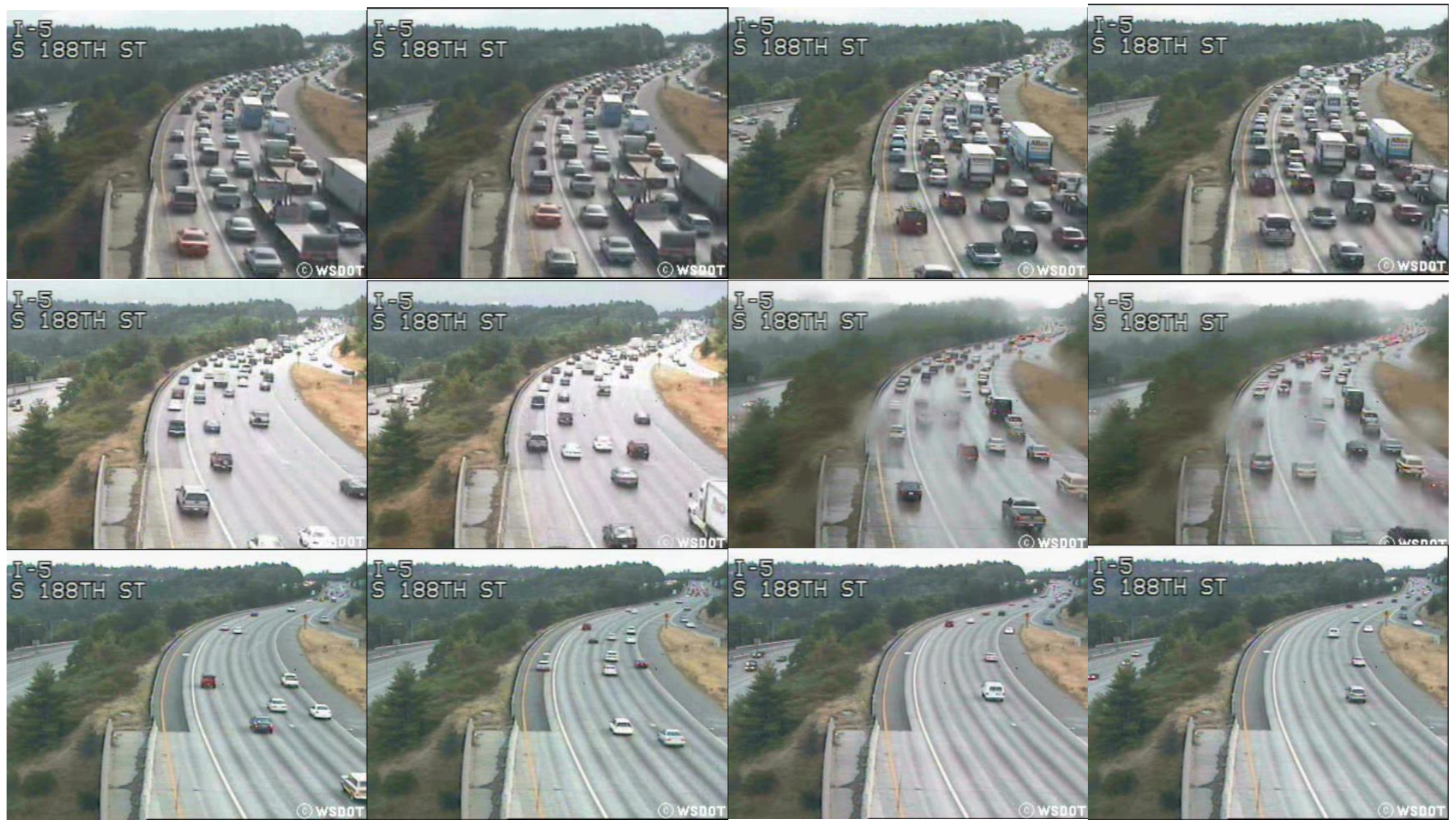}
    \end{center}
    \caption{Some samples from the Highway traffic dataset. First row illustrates heavy traffic level, second row illustrates medium traffic level and last row illustrates light traffic level.}\label{FigE5}
\end{figure*}

\begin{table*}
   \centering
   \caption{Subspace clustering results on the Traffic dataset.\label{Traffictab}}{
   \begin{tabular}{|c|c|c|c|c|c|c|c|c|}
     \hline
             Methods & GPSSVR & LapGPSSVR & GLRR-F & LRR & SSC & SCGSM & SMCE & LS3C \\
             \hline
             Traffic & \textbf{0.8617} & \textbf{0.8933} & 0.8498 & 0.6838 & 0.6285 & 0.6443 & 0.5613 & 0.6364 \\
     \hline
   \end{tabular}}
\end{table*}

The model parameter setting is described as follows. Each video clip is regarded as an image set and we generate a total of $253$ image sets labeled with 3 clusters. The Grassmann points in this experiment are chosen as $p=10$ dimension subspaces, i.e., $\mathbf X_i\in \mathcal{G}(10,576)$. We choose the expected rank and the neighbor size as $r=2$ and $C=61$. We empirically set $\lambda=2$ and $\beta=0.006$, respectively. For LRR and SSC methods, we vectorize the first $42$ frames in each clip (discarding the rest frames in the clip) and then use PCA algorithm to reduce the dimension $24\times 24 \times 42 = 24192$ to $147$.

Table \ref{Traffictab} presents the clustering result of all the algorithms. Obviously, our proposed methods get the highest accuracy $89.33\%$ which almost reaches the highest classification accuracy 92.18\% in \cite{SankaranarayananTuragaBaraniukChellappa2010}. For the real traffic applications, we can use the proposed methods to learn the different thresholds for the traffic jam levels on specific roads based on the historical traffic data. Thus it would be more accurate in classifying traffic jam levels on individual road than using the empirical uniform thresholds for all roads. So, this method is meaningful for some practical applications.

\section{Conclusion}\label{Sec:6}
In this paper, we have proposed a novel PSSVR model on Grassmann manifolds by embedding the manifold onto the space of symmetric matrices. Compared with the nuclear norm used in the LRR method, it has been proved that PSSV is a better approximation to the rank minimization problem, which is beneficial in exploring the global structure of data. We also propose efficient algorithms for the proposed methods. The computational complexity of the proposed GPSSVR method is presented, which proves that our algorithms are effective. In addition, to maintain the local structure hidden in data, we introduce a Laplacian regularization into our model. Several public video datasets are used to evaluate the performance of the proposed methods. The experimental results show that the proposed methods outperform the state-of-the-art clustering methods.

\section*{Acknowledgements}
The research project is supported by the Australian Research Council (ARC) through the grant DP140102270 and also partially supported by National Natural Science Foundation of China under Grant No. 61390510, 61672071, 61632006, 61370119, Beijing Natural Science Foundation No. 4172003, 4162010, 4152009, and Funding Project for Academic Human Resources Development in Institutions of Higher Learning Under the Jurisdiction of Beijing Municipality No.IDHT20150504.


\bibliographystyle{IEEEtran}

\section{Appendix}



Theorem \ref{Convergencethm} (CONVERGENCE). \emph{Let $S^k=(\mathbf Z^k,\mathbf J^k,\mathbf Y^k, \hat{\mathbf Y}^k)$ where $\hat{\mathbf Y}^{k+1}=\mathbf Y^k+\mu^k(\mathbf Z^k-\mathbf J^{k+1})$. If $\{\mathbf Y^k\}_{k=1}^{\infty}$ and $\{\hat{\mathbf Y}^{k}\}_{k=1}^{\infty}$ are bounded, $\lim\limits_{k\rightarrow\infty}(\mathbf Y^{k+1}-\mathbf Y^k)=0$, and $\mu^k$ is non-decreasing, then any accumulation point of $\{S^k\}_{k=1}^{\infty}$ satisfies the KKT condition. In particular, whenever $\{S^k\}_{k=1}^{\infty}$ converges, it converges to a KKT point of problem (\ref{Lap})}.

\begin{proof}

The KKT condition of problem (\ref{Lap}) is that there exists $(\mathbf Z^*,\mathbf J^*,\mathbf Y^*, \hat{\mathbf Y}^*)$ such that

\begin{equation*}
1) \ \mathbf Z^*-\mathbf J^* = 0 \nonumber; \ \ 2) \ \mathbf Y^* \in \partial\|\mathbf J^*\|_{>r}; \ \ 3) \ \mathbf Y^* \in 2\lambda\Delta - 2\lambda\mathbf Z^*\Delta
\end{equation*}

To prove Theorem 5.1, we will analyze each condition in the following sections, respectively,

1) For $\mathbf Y$. From the formula (\ref{updatingY}), we can obtain $\frac{1}{\mu^k}(\mathbf Y^{k+1}-\mathbf Y^{k}) = \mathbf Z^{k+1}-\mathbf J^{k+1}$. Since $\lim\limits_{k\rightarrow\infty}(\mathbf Y^{k+1}-\mathbf Y^k)=0$ and $\mu^k$ is non-decreasing, we have $\mathbf Z^{k+1}-\mathbf J^{k+1} = \frac{1}{\mu^k}(\mathbf Y^{k+1}-\mathbf Y^{k}) \rightarrow 0$.

2) Since $\mathbf J^{k+1}$ is obtained from the problem (\ref{LapPSSVR_subproblemJ}), by taking the derivative of $f(\mathbf Z^k, \mathbf J, \mathbf Y^k, \mathbf \mu^k)$,  we have
\begin{equation}
\begin{aligned}
0 &\in \partial\|\mathbf J^{k+1}\|_{>r}-\mathbf Y_k -\mu^k(\mathbf Z^k-\mathbf J^{k+1})  \\
  &= \partial\|\mathbf J^{k+1}\|_{>r}-\mathbf Y_k -\mu^k(\mathbf Z^{k+1}-\mathbf J^{k+1}) -\mu^k(\mathbf Z^k - \mathbf Z^{k+1})  \\
  &(\text{where} \ \ \mathbf Y^{k+1}=\mathbf Y^k + \mu^k(\mathbf Z^{k+1}-\mathbf J^{k+1})) \\
  &=\partial\|\mathbf J^{k+1}\|_{>r}-\mathbf Y^{k+1}-\mu^k(\mathbf Z^k - \mathbf Z^{k+1})  \\
  &\Rightarrow \mathbf Y^{k+1}+\mu^k(\mathbf Z^k-\mathbf Z^{k+1}) \in \partial \|\mathbf J^{k+1}\|_{>r}
\end{aligned}
\end{equation}

Due to $\{\mathbf Y^k\}_{k=1}^\infty$ and $\{\hat{\mathbf Y}^k\}_{k=1}^\infty$ are bounded, there mush be a scalar $c>0$ such that $\|\mathbf Y^{k+1}\|_F\leq c$ and $\|\hat{\mathbf Y}^{k+1}\|_F\leq c$.

\begin{equation}
\begin{aligned}
\mathbf Y^{k+1}-\hat{\mathbf Y}^{k+1} &= \mu^k(\frac{1}{\mu^k}(\mathbf Y^{k+1}-\mathbf Y^k)-(\mathbf Z^k-\mathbf J^{k+1}))  \\
&= \mu^k((\mathbf Z^{k+1}-\mathbf J^{k+1})-(\mathbf Z^k-\mathbf J^{k+1})) \\
&= \mu^k(\mathbf Z^{k+1}-\mathbf Z^k),
\end{aligned}
\end{equation}
so,
\begin{equation}
\begin{aligned}
\|\mathbf Z^{k+1}-\mathbf Z^k\|_F &= \frac{1}{\mu^k}\|\mathbf Y^{k+1}-\hat{\mathbf Y}^{k+1}\|_F  \\
&\leq \frac{1}{\mu^k}(\|\mathbf Y^{k+1}\|_F+\|\hat{\mathbf Y}^{k+1}\|_F) \\
&\leq \frac{2c}{\mu^k}\rightarrow 0. \ (\text{since} \ \mu^k \ \text{is no-decreasing})
\end{aligned}
\end{equation}

Finally, we can get $\mathbf Y^{k+1}\in\partial\|\mathbf J^{k+1}\|_{>r}$.

3) Since $\mathbf Z^{k+1}$ is obtained the problem (\ref{LapPSSVR_subproblemZ}), by taking the derivative of $f(\mathbf Z, \mathbf J^{k+1}, \mathbf Y^k, \mathbf \mu^k)$, we have
\begin{equation}
\begin{aligned}
 0 &\in -2\lambda\Delta + 2\lambda\mathbf Z^{k+1}\Delta +\mathbf Y^k + \mu^k(\mathbf Z^{k+1}-\mathbf J^{k+1}) \\
   &= -2\lambda\Delta + 2\lambda\mathbf Z^{k+1}\Delta +\mathbf Y^{k+1} \\
   &\Rightarrow \mathbf Y^{k+1} \in 2\lambda\Delta - 2\lambda\mathbf Z^{k+1}\Delta ,
\end{aligned}
\end{equation}

Therefore, our proposed method is converged. As for the detailed proof, please refer to the proof of Proposition 1 in \cite{OhTaiBazinKimKweon2015}

\end{proof}


\begin{IEEEbiography}[{\includegraphics[width=1in,height=1.25in,clip,keepaspectratio]{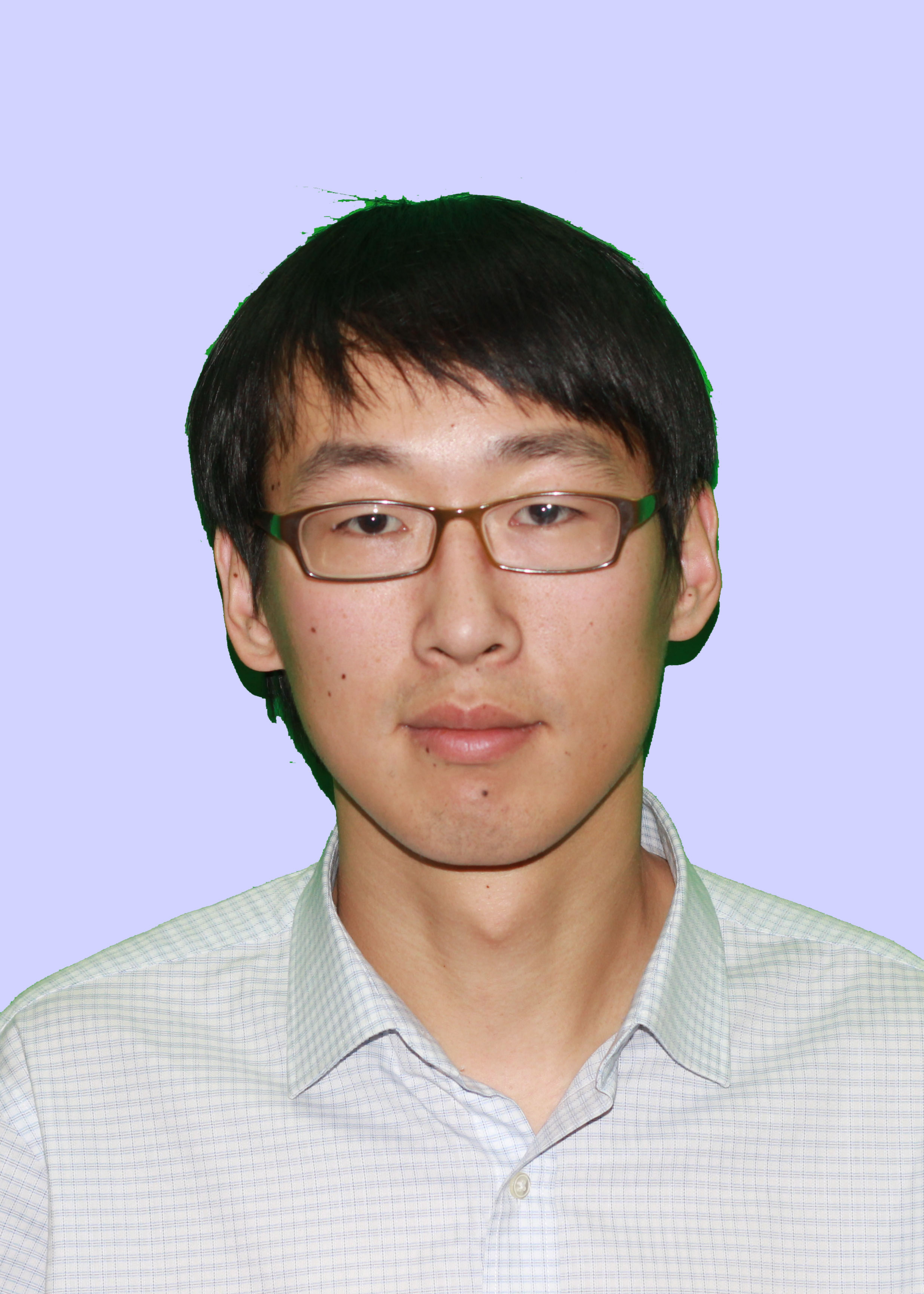}}]
{Boyue Wang} received the B.Sc. degree from Hebei University of Technology,
Tianjin, China, in 2012. he is currently pursuing the
Ph.D. degree in the Beijing Municipal Key Laboratory of Multimedia and Intelligent Software Technology,
Beijing University of Technology, Beijing.
His current research interests include computer
vision, pattern recognition, manifold learning and kernel methods.
\end{IEEEbiography}

\begin{IEEEbiography}[{\includegraphics[width=1in,height=1.25in,clip,keepaspectratio]{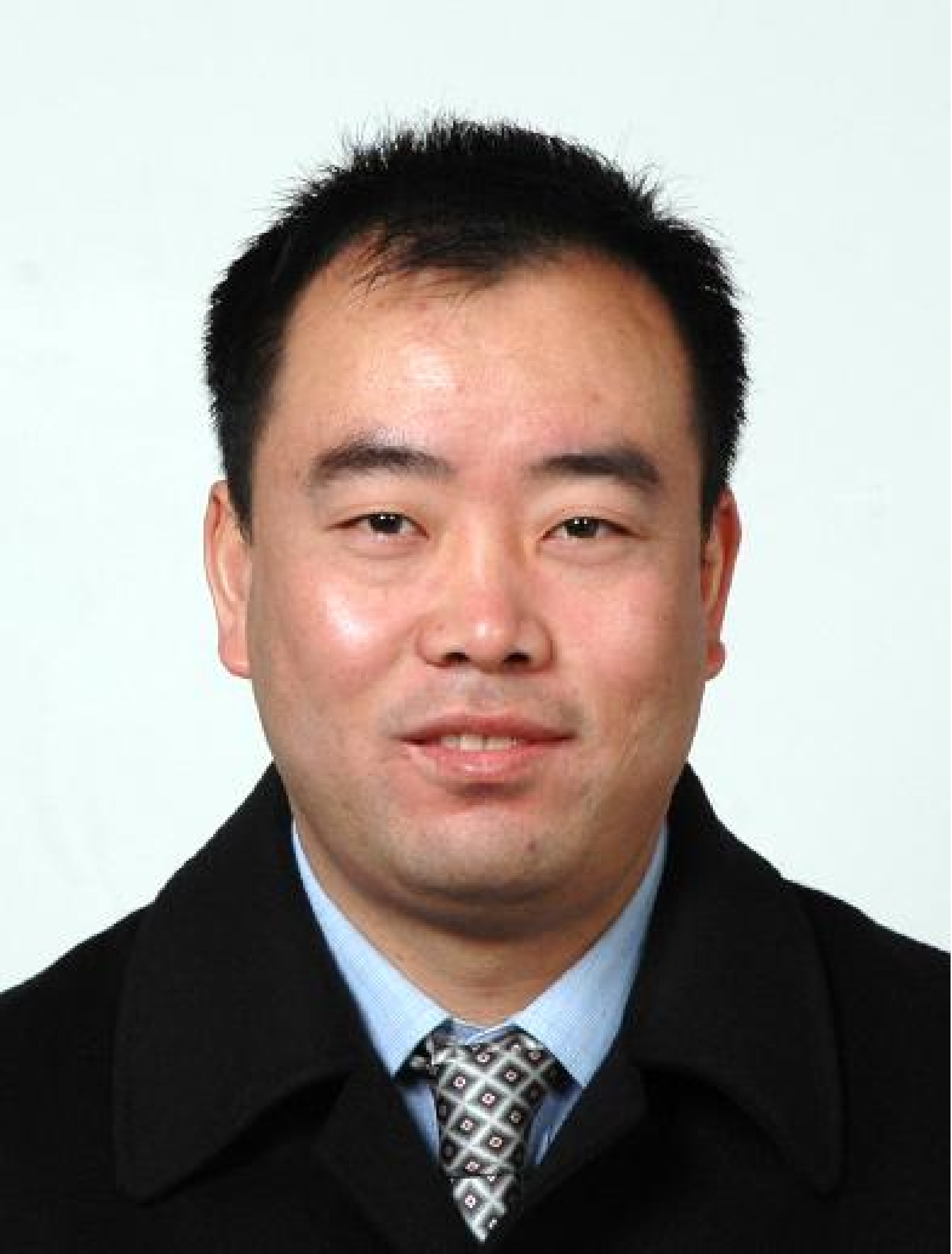}}]
{Yongli Hu} received his Ph.D. degree from Beijing University of Technology in 2005. He is a professor in College of Metropolitan Transportation at Beijing University of Technology. He is
a researcher at the Beijing Municipal Key Laboratory of Multimedia and Intelligent Software Technology.
His research interests include computer graphics, pattern recognition and multimedia technology.
\end{IEEEbiography}

\begin{IEEEbiography}[{\includegraphics[width=1in,height=1.25in,clip,keepaspectratio]{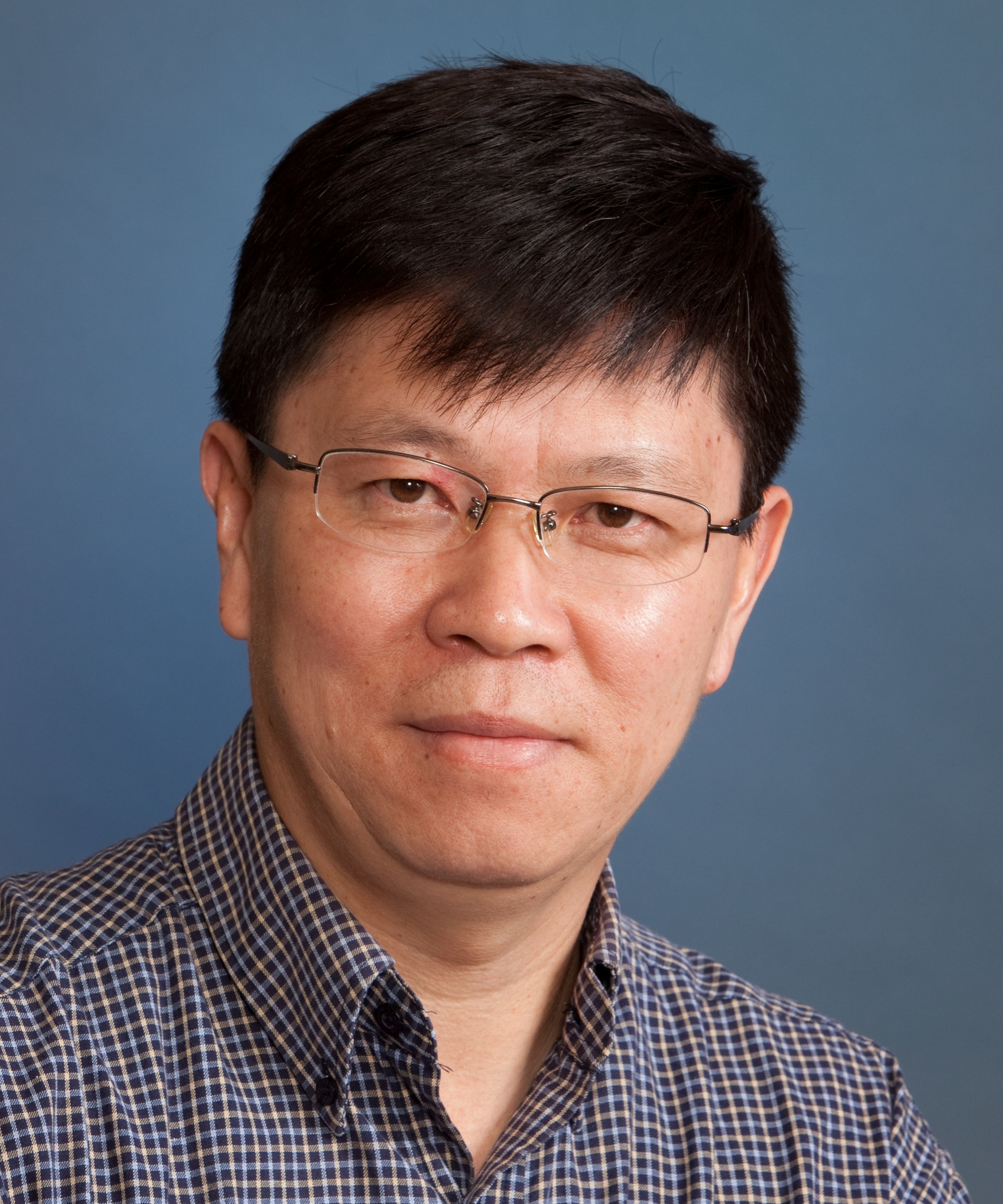}}]
{Junbin Gao} graduated from Huazhong University of Science and Technology (HUST),
China in 1982 with BSc. degree in Computational Mathematics and
obtained PhD from Dalian University of Technology, China in 1991. He is a Professor of Big Data Analytics in the University of Sydney Business School at the University of Sydney and was a Professor in Computer Science
in the School of Computing and Mathematics at Charles Sturt
University, Australia. He was a senior lecturer, a lecturer in Computer Science from 2001 to 2005 at
University of New England, Australia. From 1982 to 2001 he was an
associate lecturer, lecturer, associate professor and professor in
Department of Mathematics at HUST. His main research interests
include machine learning, data analytics, Bayesian learning and
inference, and image analysis.
\end{IEEEbiography}

\begin{IEEEbiography}[{\includegraphics[width=1in,height=1.25in,clip,keepaspectratio]{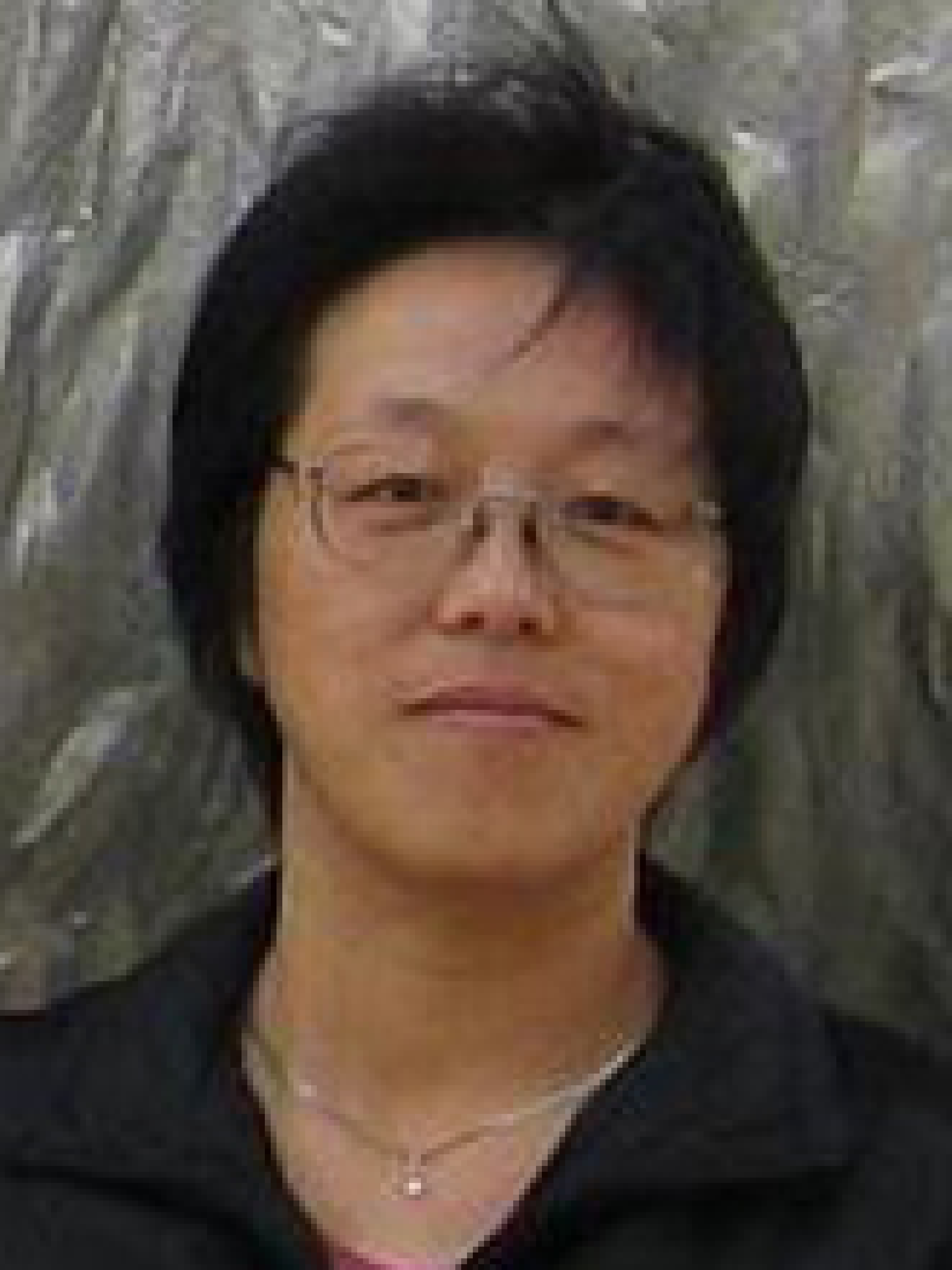}}]
{Yanfeng Sun} received her Ph.D. degree from Dalian University of Technology in 1993. She is a professor in College of Metropolitan Transportation at Beijing University of Technology. She is
a researcher at the Beijing Municipal Key Laboratory of Multimedia and Intelligent Software Technology. She is the membership of China Computer Federation.
 Her research interests are multi-functional perception and image processing.
\end{IEEEbiography}

\begin{IEEEbiography}[{\includegraphics[width=1in,height=1.25in,clip,keepaspectratio]{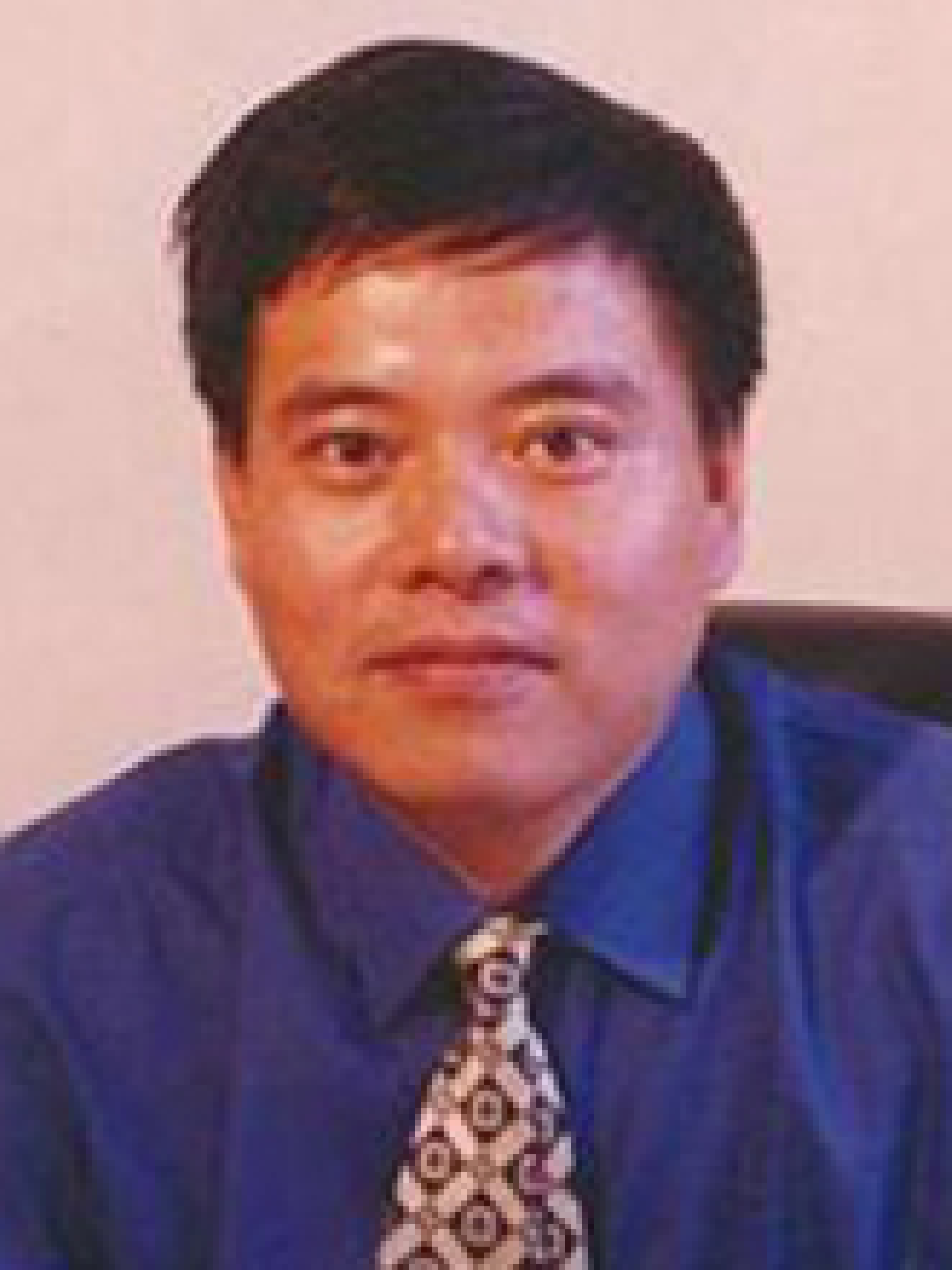}}]
{Baocai Yin} received his Ph.D. degree from Dalian University of Technology in 1993.
He is a Professor in the College of Computer Science and Technology, Faculty of Electronic Information and Electrical Engineering, Dalian University of Technology. He is
a researcher at the Beijing Municipal Key Laboratory of Multimedia and Intelligent Software Technology.
He is a member of China Computer Federation. His
research interests cover multimedia, multifunctional perception, virtual reality and computer graphics.
\end{IEEEbiography}
\vfill

\end{document}